\PassOptionsToPackage{pagebackref,breaklinks,colorlinks,citecolor=eccvblue}{hyperref}
\PassOptionsToPackage{pagebackref,breaklinks,colorlinks,citecolor=eccvblue}{hyperref}
\documentclass[runningheads]{llncs}

\usepackage{eccv}

\usepackage{amsmath,amsfonts,bm}

\def\eqref#1{equation~\ref{#1}}
\def\1{\bm{1}}

\def\rvx{{\mathbf{x}}}

\DeclareMathAlphabet{\mathsfit}{\encodingdefault}{\sfdefault}{m}{sl}
\SetMathAlphabet{\mathsfit}{bold}{\encodingdefault}{\sfdefault}{bx}{n}

\newcommand{\cmark}{\ding{51}} % checkmark
\definecolor{bestred}{RGB}{220, 235, 255}   % stronger light blue
\definecolor{secondred}{RGB}{245, 250, 255} % very light blue

\newcommand{\modelname}{\textsc{EFlow}\xspace}  % adjust name + spacing

\usepackage{eccvabbrv}
\usepackage[utf8]{inputenc} % allow utf-8 input
\usepackage[T1]{fontenc}    % use 8-bit T1 fonts
\usepackage{hyperref}       % hyperlinks
\usepackage{url}            % simple URL typesetting
\usepackage{booktabs}       % professional-quality tables
\usepackage{amsfonts}       % blackboard math symbols
\usepackage{nicefrac}       % compact symbols for 1/2, etc.
\usepackage{microtype}      % microtypography
\usepackage{xcolor}         % colors
\usepackage[table]{xcolor}
\usepackage{multirow}
\usepackage{graphics}
\usepackage{graphicx}
\usepackage{subcaption}
\usepackage{pifont}
\usepackage{enumitem}
\usepackage{arydshln}
\usepackage[symbol]{footmisc}

\usepackage{algorithm}
\usepackage{algorithmicx}
\usepackage{algpseudocode}
\usepackage{amssymb}
\usepackage{amsmath}
\usepackage{wrapfig}
\usepackage{makecell}
\usepackage{siunitx}
\usepackage{etoc}

\usepackage[accsupp]{axessibility}  % Improves PDF readability for those with disabilities.

\usepackage{hyperref}
\usepackage{orcidlink}

\begin{document}

% \title{ESF: Fast Training Few-Step Video Generator \\ from Scratch via Efficient Solution Flow} 
% \title{\modelname: Fast Training Few-Step Video Generator \\ from Scratch via Efficient Solution Flow} 
\title{\modelname: Fast Few-Step Video Generator Training \\ from Scratch via \underline{E}fficient Solution \underline{F}low} 

\titlerunning{EFlow: Fast Few-Step Video Generator Training from Scratch}

\author{Dogyun Park\inst{1,2}\thanks{Work done during internship at Snap Inc.} \and
Yanyu Li\inst{1} \and
Sergey Tulyakov \inst{1} \and
Anil Kag \inst{1}
}

\authorrunning{D. Park et al.}

\institute{Snap Inc. \and
Korea University, Seoul, Republic of Korea}
\maketitle

\begin{center}
  \includegraphics[width=\linewidth]{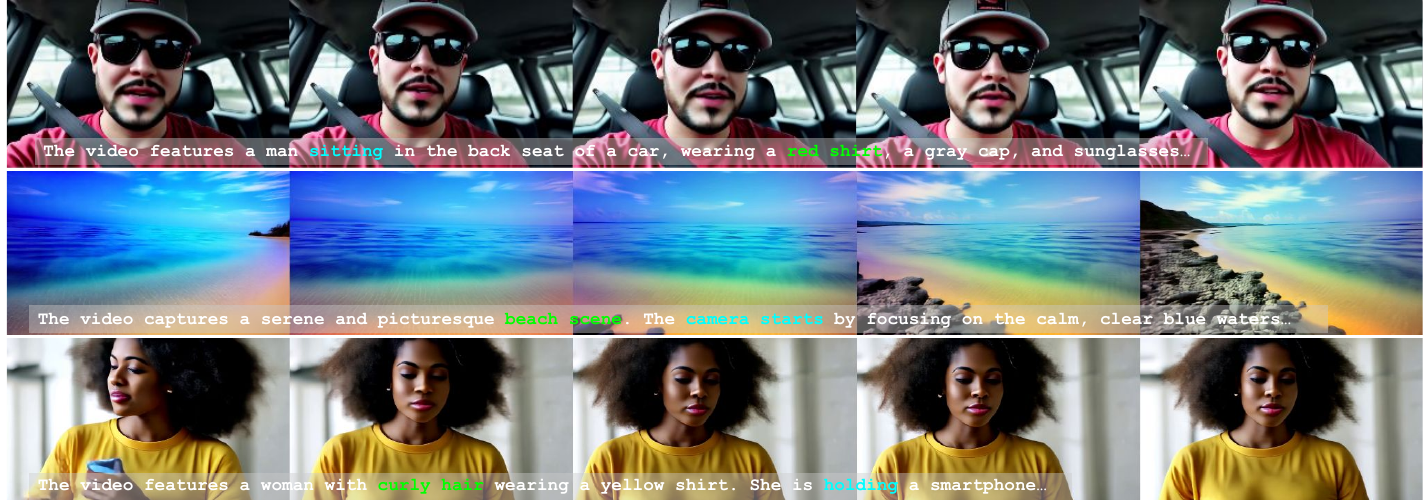}
  \\
  \vspace{-0.7em}
  \captionof{figure}{\textbf{Text-to-video results generated by our \modelname} with 4 inference steps.}
\end{center}

\begin{abstract}
  Scaling video diffusion transformers is fundamentally bottlenecked by two compounding costs: the expensive quadratic complexity of attention per step, and the iterative sampling steps. In this work, we propose \modelname, a efficient few-step training framework, that tackles these bottlenecks simultaneously. To reduce sampling steps, we build on a solution-flow objective that learns a function mapping a noised state at time $t$ to time $s$. Making this formulation computationally feasible and high-quality at video scale, however, demands two complementary innovations. First, we propose \emph{Gated Local–Global Attention}, a token-droppable hybrid block which is efficient, expressive, and remains highly stable under aggressive random token-dropping, substantially reducing per-step compute. Second, we develop an efficient few-step training recipe. We propose \emph{Path-Drop Guided} training to replace the expensive guidance target with a computationally cheap, weak path. Furthermore, we augment this with a \emph{Mean-Velocity Additivity} regularizer to ensure high fidelity at extremely low step counts. Together, our \modelname enables a practical from-scratch training pipeline, achieving up to $2.5\times$ higher training throughput over standard solution-flow, and $45.3\times$ lower inference latency than standard iterative models with competitive performance on Kinetics and large-scale text-to-video datasets.
\end{abstract}

\section{Introduction}

% Diffusion and flow-based models~\cite{wan2025wan,OpenSora,CogVideoX} enable high-fidelity video generation, but scaling them to long, high-resolution clips remains prohibitively expensive. The end-to-end cost of video diffusion transformers (VideoDiTs) suffers from a compounding double bottleneck: (i) heavy per-step compute, as spatiotemporal self-attention scales quadratically with token count, and (ii) iterative sampling, requiring tens of denoising steps that multiply this per-step cost. Together, these bottlenecks severely restrict practical training and deployment.
Diffusion and flow-based models~\cite{wan2025wan,OpenSora,CogVideoX,sora,kling} enable high-fidelity video generation, yet scaling them to long, high-resolution clips remains computationally daunting. Video diffusion transformers (VideoDiTs) suffers from a compounding double bottleneck: (i) heavy per-step compute, as spatiotemporal self-attention scales quadratically with token count, and (ii) iterative sampling, requiring tens of denoising steps that multiply this per-step cost. Together, these constraints severely limit the practicality of training and deploying video models.

Most existing research attacks these challenges in isolation. Architectural efforts reduce per-step complexity via efficient linear attention~\cite{xie2024sana,chen2025sana,zhu2025dig,zhao2026s2ditsandwichdiffusiontransformer}, sparsity~\cite{zhang2025sla,zhang2025vsa,zhang2025fast,xi2025sparse}, or token pruning~\cite{li2024vidtome,choi2024vid,park2025sprint}. However, aggressive token removal often harms stability or expressivity—especially in the presence of locality mechanisms (like convolutions) that implicitly assume dense, grid-like spatial neighborhoods. Conversely, sampling-focused approaches reduce the number of steps through distillation~\cite{dmd,dmd2,zhang2024sf,park2025blockwise} or consistency-style objectives~\cite{song2023consistency,wang2024phased,zheng2025large,geng2025mean,luo2025soflow,park2024constant}, but frequently introduce substantial engineering complexity (\eg, pre-trained teachers, multi-stage pipelines, or adversarial components). Furthermore, Classifier-Free Guidance (CFG)—a key ingredient for controllable generation—can effectively double training compute by requiring both conditional and unconditional network evaluations per step.

\begin{figure}[t]
  \centering
  \begin{minipage}[t]{0.47\linewidth}
    \centering
    \includegraphics[width=\linewidth]{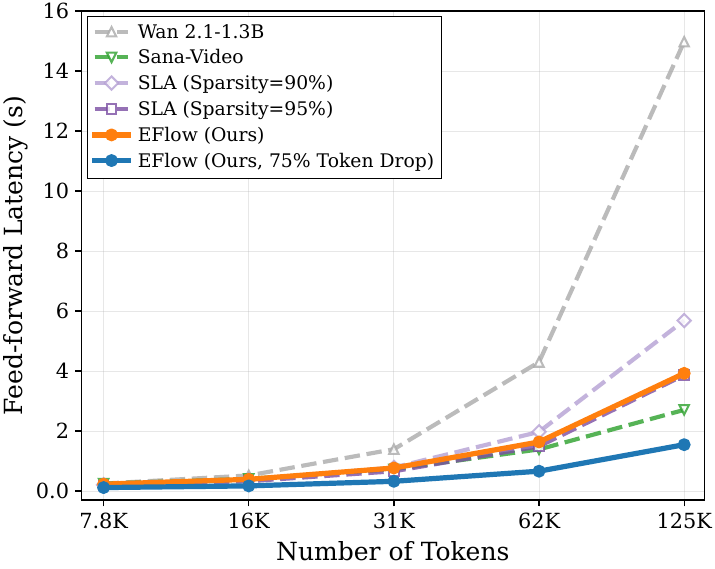}
    \vspace{-0.3em}
    {\small (a) Latency vs. tokens}
  \end{minipage}\hfill
  \begin{minipage}[t]{0.51\linewidth}
    \centering
    \includegraphics[width=\linewidth]{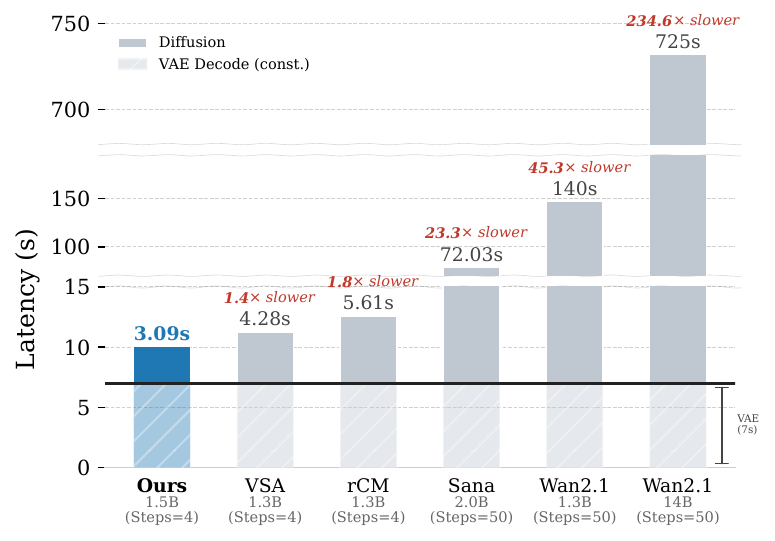}
    \vspace{-0.3em}
    {\small (b) 480p inference latency}
  \end{minipage}
  \caption{\textbf{(a) Feed-forward latency vs.\ token number} for different attention designs. Softmax attention (Wan 2.1) scales poorly as the sequence length grows while our GLGA (\modelname) is significantly faster than softmax and remains competitive with sparse/linear baselines. Applying 75\% token dropping on GLGA yields the lowest latency and the most favorable scaling. \textbf{(b) 480p inference latency} for our model and recent video generators. Our model achieves 1.4$\times$ and 45.3$\times$ faster diffusion inference over baselines by combining an efficient backbone with few-step sampling.}
  \label{fig:efficiency_comparison}
  \vspace{-1em}
\end{figure}

\begin{table}[t]
\centering
\caption{\textbf{Comparison of our method with VideoDiT families.} We summarize attention complexity, training pipeline requirements (reliance on teacher or auxiliary models like discriminator), and inference cost via number of function evaluations (NFEs). $N$, $N'$ $D$, and $W$ denote token count, sparse token number, dimension, and local window size, respectively. The ``$\times 2$'' indicates model evaluations for classifier-free guidance.}
\setlength{\tabcolsep}{7pt}
\renewcommand{\arraystretch}{1.18}
\small
\resizebox{\linewidth}{!}{%
\begin{tabular}{lccccc}
\toprule
& \makecell[c]{\textbf{Standard}\\[-1pt]{\footnotesize (Wan 2.1, Wan2.2)}}
& \makecell[c]{\textbf{Efficient}\\[-1pt]{\footnotesize (Sana, S$^2$DiT)}}
& \makecell[c]{\textbf{Efficient}\\[-1pt]{\footnotesize (VSA, SLA)}}
& \makecell[c]{\textbf{Few-step}\\[-1pt]{\footnotesize (DMD, rCM)}}
& \makecell[c]{\textbf{Ours}\\[-1pt]{\footnotesize (\modelname)}} \\
\midrule

\rowcolor{gray!15}
\multicolumn{6}{l}{\textit{Attention architecture}} \\
\textbf{Type}
& Softmax
& Linear+Conv
& Sparse softmax
& Softmax
& GLGA \\

\textbf{Complexity}
& \textcolor{Red}{$\mathcal{O}(N^2)$}
& \textcolor{ForestGreen}{$\mathcal{O}(ND)$}
& \textcolor{ForestGreen}{$\mathcal{O}(N'^2)$}
& \textcolor{Red}{$\mathcal{O}(N^2)$}
& \textcolor{ForestGreen}{$\mathcal{O}(N(D+W))$} \\

\textbf{Token-Droppable}
& \textcolor{ForestGreen}{\ding{51}}
& \textcolor{Red}{\ding{55}}
& \textcolor{Red}{\ding{55}}
& \textcolor{ForestGreen}{\ding{51}}
& \textcolor{ForestGreen}{\ding{51}} \\

\rowcolor{gray!15}
\multicolumn{6}{l}{\textit{Training pipeline}} \\
\textbf{No Teacher Model}
& \textcolor{ForestGreen}{\ding{51}}
& \textcolor{ForestGreen}{\ding{51}}
& \textcolor{Red}{\ding{55}}
& \textcolor{Red}{\ding{55}}
& \textcolor{ForestGreen}{\ding{51}} \\

\textbf{No Auxiliary Model}
& \textcolor{ForestGreen}{\ding{51}}
& \textcolor{ForestGreen}{\ding{51}}
& \textcolor{ForestGreen}{\ding{51}}
& \textcolor{Red}{\ding{55}}
& \textcolor{ForestGreen}{\ding{51}} \\

\rowcolor{gray!15}
\multicolumn{6}{l}{\textit{Inference}} \\
\textbf{Inference NFEs}
& \textcolor{Red}{$50 \times 2$}
& \textcolor{Red}{$50 \times 2$}
& \textcolor{Red}{$50 \times 2$}
& \textcolor{ForestGreen}{$4$}
& \textcolor{ForestGreen}{\textbf{$4$}} \\
\bottomrule
\end{tabular}%
}
\label{tab:paradigm_comparison}
\vspace{-1em}
\end{table}

In this work, we propose \modelname, \textbf{fast few-step VideoDiT training} (see Tab.~\ref{tab:paradigm_comparison}), by combining an efficient, token-droppable attention backbone with a scalable training recipe for solution-flow (SoFlow)~\cite{luo2025soflow}. Concretely, we build on the recent SoFlow~\cite{luo2025soflow} formulation, which learns a bi-time function mapping a noised state at time $t$ to state at any time $s$. 
% We select this framework because it offers a highly attractive paradigm for few-step generation: it enables single-stage training from scratch, completely eliminating the restrictive need for pre-trained teachers~\cite{zheng2025large} or auxiliary networks~\cite{yin2024one}. 
We select this framework for its unique advantages in few-step generation: it enables single-stage training from scratch, completely eliminating the need for pre-trained teachers~\cite{zheng2025large} or auxiliary networks~\cite{yin2024one}. 
Furthermore, its hybrid flow-matching and solution-consistency objective avoids the memory-intensive Jacobian-Vector Products (JVPs) required by consistency-style models~\cite{lu2024simplifying,geng2025mean}, permitting the use of highly optimized kernels like FlashAttention~\cite{dao2022flashattention,dao2023flashattention}. 
% allowing us to fully leverage highly optimized attention backends (\eg, FlashAttention~\cite{dao2022flashattention,dao2023flashattention}). 
%However, scaling this elegant formulation to video presents a severe computational obstacle: enforcing solution consistency across time pairs and applying classifier-free guidance requires multiple full-network evaluations per training step. When multiplied by the heavy quadratic cost of attention, this multi-pass requirement makes video-scale training prohibitively expensive.
However, scaling this formulation to video poses a severe computational challenge. Enforcing solution consistency across time pairs and applying CFG requires multiple full-network evaluations per training step. When compounded by the attention quadratic complexity, this multi-pass requirement makes video-scale training prohibitively expensive.

We introduce two synergistic innovations for scalable few-step video training:

\noindent\emph{(i) Efficient, Token-Droppable Hybrid Attention for VideoDiT.}
Prior linear-attention DiTs~\cite{chen2025sana,zhao2026s2ditsandwichdiffusiontransformer} use convolutional mixing for locality, which breaks under aggressive token dropping due to its reliance on a dense spatial grid. We instead propose \emph{Gated Local–Global Attention} (GLGA), a hybrid block fusing global linear attention with local sliding-window attention. Because window attention operates over available tokens rather than fixed grids, it is inherently robust to severe token removal. This adaptive routing yields an expressive and highly stable VideoDiT even when dropping the vast majority of tokens. Empirically, on the Kinetics dataset, this hybrid design achieves a \textbf{13\% improvement} in Fréchet Video Distance (FVD) compared to linear-attention~\cite{chen2025sana} (Tab.~\ref{tab:fvd_comparison}).

\noindent\emph{(ii) Efficient Few-Step Training Recipe.}
Integrating CFG into solution-flow requires two network evaluations per sample—a steep computational cost for video. We propose \emph{Path-Drop Guided} (PDG) training, which replaces the expensive unconditional branch with a computationally cheap ``weak-path'' forward pass obtained by skipping middle transformer layers. Furthermore, to preserve quality at few-step sampling, we introduce a global \emph{Mean-Velocity Additivity} (MVA) regularizer that mitigates the integration errors in local consistency matching. 

%token-dropping and PDG 
Integrating the \emph{GLGA VideoDiT} with \emph{Scalable Few-Step Recipe}, \modelname compounds compute savings to make from-scratch, few-step video training practical. Our framework achieves \textbf{1.76$\times$ and 2.5$\times$ higher training throughput} on 480p video compared to flow matching and solution-flow matching, respectively (Fig.~\ref{fig:efficiency}(a)). These gains extend to deployment, where \modelname achieves \textbf{1.4$\times$ to 45.3$\times$ lower inference latency} (Fig.~\ref{fig:efficiency_comparison}) than recent few-step methods (VSA, rCM), efficient architectures (Sana), and standard models (Wan2.1 1.3B), while maintaining competitive VBench scores on 480p text-to-video generation.

% \noindent\textbf{(ii) Efficient Few-Step Training Recipe.}
% Integrating CFG into solution-flow typically doubles the computational cost for video. We propose \textbf{Path-Drop Guided (PDG)} training, which replaces the expensive unconditional branch with a computationally cheap ``weak-path'' forward pass obtained by skipping middle transformer layers. Furthermore, to preserve quality at low step counts, we introduce a global \textbf{Mean-Velocity Additivity (MVA)} regularizer that mitigates the integration errors inherent in local consistency matching. Together, token-dropping and PDG compound compute savings, making from-scratch, few-step video training practical. Our framework achieves \textbf{1.76$\times$ and 2.5$\times$ higher training throughput} on 480p video compared to standard flow matching and solution-flow matching, respectively (\cref{fig:efficiency}).

% Ultimately, our end-to-end framework delivers profound efficiency gains at deployment. We achieve \textbf{1.4$\times$ to 45.3$\times$ lower inference latency} (\cref{fig:efficiency_comparison}) than recent few-step distillation methods (VSA, rCM), efficient architectures (Sana), and standard models (Wan 2.1), while maintaining competitive VBench scores on 480p text-to-video generation.

In summary, our key contributions are:
\begin{itemize}
\item \textbf{Efficient, Droppable VideoDiT Architecture:} We propose Gated Local–Global Attention (GLGA), which maintains high fidelity under aggressive token dropping, resulting in substantial training and inference speedups.
\item \textbf{Path-Drop Guided (PDG) Training:} We introduce a novel guidance mechanism that constructs guidance targets using a lightweight, block-skipped forward pass, drastically reducing the CFG compute overhead.
%compute overhead of classifier-free guidance.
\item \textbf{Large-Scale Few-Step Video Training:} We present a simple, teacher-free pipeline that incorporates a novel global long-jump consistency objective (MVA) to enable high-quality generation at extremely low step counts, achieving performance competitive with recent state-of-the-art open-source models.
\end{itemize}

\section{Related Work}
\noindent{\textbf{Efficient architectures and few-step generation.}}
Diffusion Transformers (DiTs)~\cite{peebles2023scalable} have demonstrated that transformer backbones are highly competitive, yet their computational cost scales quadratically with token count. To mitigate this per-step overhead, prior work either accelerates exact attention (\eg, FlashAttention~\cite{dao2022flashattention,dao2023flashattention}) or replaces softmax attention with sub-quadratic variants like linear or gated-linear attention (\eg, Sana, DiG, Sana-Video, and S$^2$DiT)~\cite{xie2024sana,zhu2025dig,chen2025sana,zhao2026s2ditsandwichdiffusiontransformer}. While sparse attention methods~\cite{zhang2025fast,zhang2025sla,Xia_2025_ICCV,xi2025sparse} further reduce complexity, they often follow a ``dense-to-sparse'' paradigm, requiring a pre-trained dense model as a high-fidelity starting point. This reliance on pre-existing weights poses a significant barrier to truly efficient, from-scratch video training.

Orthogonally, few-step generation methods address the iterative sampling bottleneck by drastically reducing the number of function evaluations (NFEs). A prominent line of work achieves this by distilling pre-trained teachers via Consistency Models~\cite{song2023consistency,lu2024simplifying,zheng2025large}, Distribution Matching Distillation (DMD)~\cite{yin2024one}, or Score Identity Distillation (SiD)~\cite{zhou2024score}. While these methods enable one- or few-step generation, they introduce substantial engineering complexity---via multi-stage pipelines and a strict reliance on expensive pre-trained teachers. To overcome these constraints, recent advancements such as MeanFlow~\cite{geng2025mean} and SoFlow~\cite{luo2025soflow} learn few-step generative dynamics directly from scratch. SoFlow, in particular, parameterizes a bi-time solution function trained through a combination of flow-matching and solution-consistency objectives. Our work bridges these efficient architecture and teacher-free few-step generation tracks: we adopt SoFlow to eliminate distillation overhead and introduce a novel token-droppable VideoDiT backbone alongside path-drop guided (PDG) training to make this paradigm highly scalable at video resolutions.

\noindent{\textbf{Token dropping for diffusion transformers.}}
Reducing sequence length is the most direct path to lowering attention costs. Token merging (\eg, ToMe~\cite{bolya2022token}) has been widely adapted for diffusion inference to preserve visual quality while reducing FLOPs. Recent extensions to video (\eg, VidToMe~\cite{li2024vidtome} and vid-TLDR~\cite{choi2024vid}) merge temporally or spatially redundant tokens to ensure multi-frame consistency without heavy compute. Alternatively, importance-aware pruning (\eg, ToME-SD~\cite{bolya2023token}) selectively discards less influential tokens based on attention scores or classifier-free guidance magnitudes. While inference-time reduction is valuable, random token dropping during training is essential to lower end-to-end training costs. However, naive dropping destroys grid locality and induces a severe train--inference mismatch. Existing attempts to mitigate this mismatch rely on masked autoencoding auxiliary tasks (\eg, MaskDiT~\cite{zheng2023fast}, MDTv2~\cite{gao2023mdtv2}) or sparse--dense residual fusion (\eg, SPRINT~\cite{park2025sprint}). 

Crucially, while standard DiTs can tolerate token dropping, their quadratic complexity remains a concern. Conversely, recent efficient architectures replace softmax with linear attention but often rely on convolutions to recover local expressivity~\cite{xie2024sana,chen2025sana,zhu2025dig}. Because convolutions assume a dense, regular spatial grid, these models fundamentally degrade under aggressive random token removal. To resolve this conflict, our Gated Local--Global Attention (GLGA) replaces rigid convolutions with local sliding-window attention. Since window attention operates over available tokens rather than fixed grids, it preserves linear efficiency and local expressivity while remaining inherently robust to token dropping.

\section{Preliminary}

\subsection{Solution Flow Matching}
Flow models~\cite{ho2020denoising,song2020denoising,park2024ddmi,lipmanflow,liu2022flow} specify a transport between
$\pi_0$ and a reference $\pi_1$ using a velocity ODE, $\frac{\mathrm{d}x_t}{\mathrm{d}t} = v(x_t, t,c),$ where $t\in[0,1]$ and $c$ is condition such as text. Following~\cite{ma2024sit}, we use the linear Gaussian coupling
$x_t \sim \mathcal{N}(\alpha_t x_0,\sigma_t^2 I)$ with $\alpha_t=1-t$ and $\sigma_t=t$.
% A network $v_\theta$ is trained by flow matching, \ie, regressing to the ground-truth conditional velocity induced by the coupling, \ie, $v(x_t,t,c|x_0)=x_1-x_0$:
A network $v_\theta$ is trained by flow matching to regress the induced ground-truth conditional velocity, \ie, $v(x_t,t,c|x_0)=x_1-x_0$:
\begin{equation}
    \min_\theta \;\mathbb{E}_{x_0, x_1, t}\!\left[\left\|v_\theta(x_t,t,c) - v(x_t,t,c\,|\,x_0)\right\|_2^2\right].
    \label{eq:loss}
\end{equation}
% SoFlow~\cite{luo2025soflow} generalizes velocity learning by directly modeling the \emph{solution map} $f(\rvx_t,t,s,c)$ for any $0\le s\le t\le 1$, which returns the ODE solution at time $s$ when initialized from state
SoFlow~\cite{luo2025soflow} generalizes this by directly modeling the \emph{solution map} $f(\rvx_t,t,s,c)$ for any $0\le s\le t\le 1$, which returns the ODE solution at time $s$ given the state
$\rvx_t$ at time $t$ under conditioning $c$.
% Under standard regularity assumptions, the ground-truth solution satisfies the boundary condition
Under regularity assumptions, the ground-truth solution satisfies the boundary condition
$f(\rvx_t,t,t,c)=\rvx_t$ and the consistency relation $\partial_s f(\rvx_t,t,s,c)=v(f(\rvx_t,t,s,c),s,c)$.
SoFlow shows it suffices to learn a parameterized $f_\theta$ that satisfies the boundary condition and an associated PDE.
% To enforce the boundary condition by construction, we use an Euler-style parameterization
To enforce the boundary condition, we use an Euler-style parameterization
\begin{equation}
    f_\theta(\rvx_t,t,s,c) = \rvx_t + (s-t)\,F_\theta(\rvx_t,t,s,c),
\label{eq:parameterization}
\end{equation}
where $F_\theta$ is implemented by a DiT backbone. 

\noindent{\textbf{{Flow-matching objective.}}}
% Setting $s=t$ in the SoFlow formulation yields $v(\rvx_t,t,c)=\partial_s f_\theta(\rvx_t,t,t,c)$.
%Under the Euler parameterization, $\partial_s f_\theta(\rvx_t,t,t,c)=F_\theta(\rvx_t,t,t,c)$, giving the FM loss
At boundary $s=t$, we obtain velocity $v(\rvx_t,t,c)=\partial_s f_\theta(\rvx_t,t,t,c)$.
From Eq.~\ref{eq:parameterization}, $\partial_s f_\theta(\rvx_t,t,t,c)=F_\theta(\rvx_t,t,t,c)$, giving the FM loss
\begin{equation}
\mathcal{L}_{\text{FM}}
= \mathbb{E}_{\rvx_0,\rvx_1,t}
\left[
    w_{\text{FM}}(t)\,
    \left\|
        F_\theta(\rvx_t,t,t,c) - v(\rvx_t, t, c\,|\,x_0)
    \right\|_2^2
\right],
\label{eq:soflow_fm}
\end{equation}
where $w_{\text{FM}}(t) = 1/(\text{MSE}+\epsilon)^p$; MSE indicates original mean squared error, $\epsilon$ is a factor for numerical stability, and $p$ is a factor that determines the robustness. 

\noindent{\textbf{{Solution consistency matching.}}}
For $s<t$, SoFlow enforces a finite solution-consistency constraint over an intermediate time $l$ with $s\le l\le t$.
Using a first-order Taylor approximation to the intermediate state,
$\rvx_l \approx \rvx_t + v(\rvx_t, t, c\,|\,x_0)(l-t)$,
and a stop-gradient target $f_{\theta^-}$, the solution-consistency matching (SCM) loss is
\begin{equation}
\mathcal{L}_{\text{SCM}}
= \mathbb{E}
\left[
    w_{\text{SCM}}\,
    \big\|
        f_\theta(\rvx_t,t,s,c)
        -
        f_{\theta^{-}}\!\big(\rvx_t + v(\rvx_t, t,c\,|\,x_0)(l-t),\, l,\, s\big)
    \big\|_2^2
\right],
\label{eq:soflow_scm}
\end{equation}
where $w_{\text{SCM}}(t) = \frac{1}{(t-l)(t-s)}\times \frac{1}{(\frac{\text{MSE}}{(t-l)^2}+\epsilon)^p}$ following \cite{luo2025soflow}.
The training loss is $\mathcal{L} = \lambda\,\mathcal{L}_{\text{FM}}+(1-\lambda)\,\mathcal{L}_{\text{SCM}}$, where $\lambda$ controls the fraction of a data batch dedicated to $\mathcal{L}_{\text{FM}}.$ 
In the following sections, we introduce architectural and training innovations to make this solution-flow paradigm scalable and robust for high-resolution video generation.
% In our work, we build on this formulation and improve on the architectural and training mechanisms needed to make solution-flow training practical and better at video. 
% In our work, we build on this formulation and improve on the architectural and training mechanisms needed to make solution-flow training practical and better at video. The total training loss is $\mathcal{L} = \lambda\,\mathcal{L}_{\text{FM}}+(1-\lambda)\,\mathcal{L}_{\text{SCM}}$, where $\lambda$ controls the fraction of a data batch dedicated to $\mathcal{L}_{\text{FM}}.$ 

% \noindent{\textbf{Training with token-dropping.}}
\subsection{Training with Token-Dropping.}
Token-dropping reduces the quadratic attention cost in DiTs by shortening the token sequence. For a drop ratio $r$, we
remove $\lfloor rN\rfloor$ tokens and process only the remaining subset, yielding substantial compute savings.
Recently, SPRINT~\cite{park2025sprint} stabilized training with high-ratio dropping (\eg, 75\%) by restructuring the transformer into three stages—an encoder, a
deep middle stack, and a decoder—and connecting encoder and decoder with a long dense residual path.

Concretely, SPRINT computes dense shallow features, applies token dropping before the middle blocks, and runs middle stack only on the surviving sparse tokens. To fuse sparse and dense features, it lifts the sparse output back to $N$ by padding dropped positions with a fixed \texttt{[MASK]} token, and fuses it with the dense shallow features via long residual connection.
Training follows a two-stage schedule: long high-drop pre-training, followed by short full-token fine-tuning where the
middle blocks switch back to dense inputs to close the train--inference gap. In our method, we adopt this SPRINT's design, training schedule, and \emph{apply token-dropping only to the middle blocks}.
\begin{figure}[t]
    \centering
    \includegraphics[width=\linewidth]{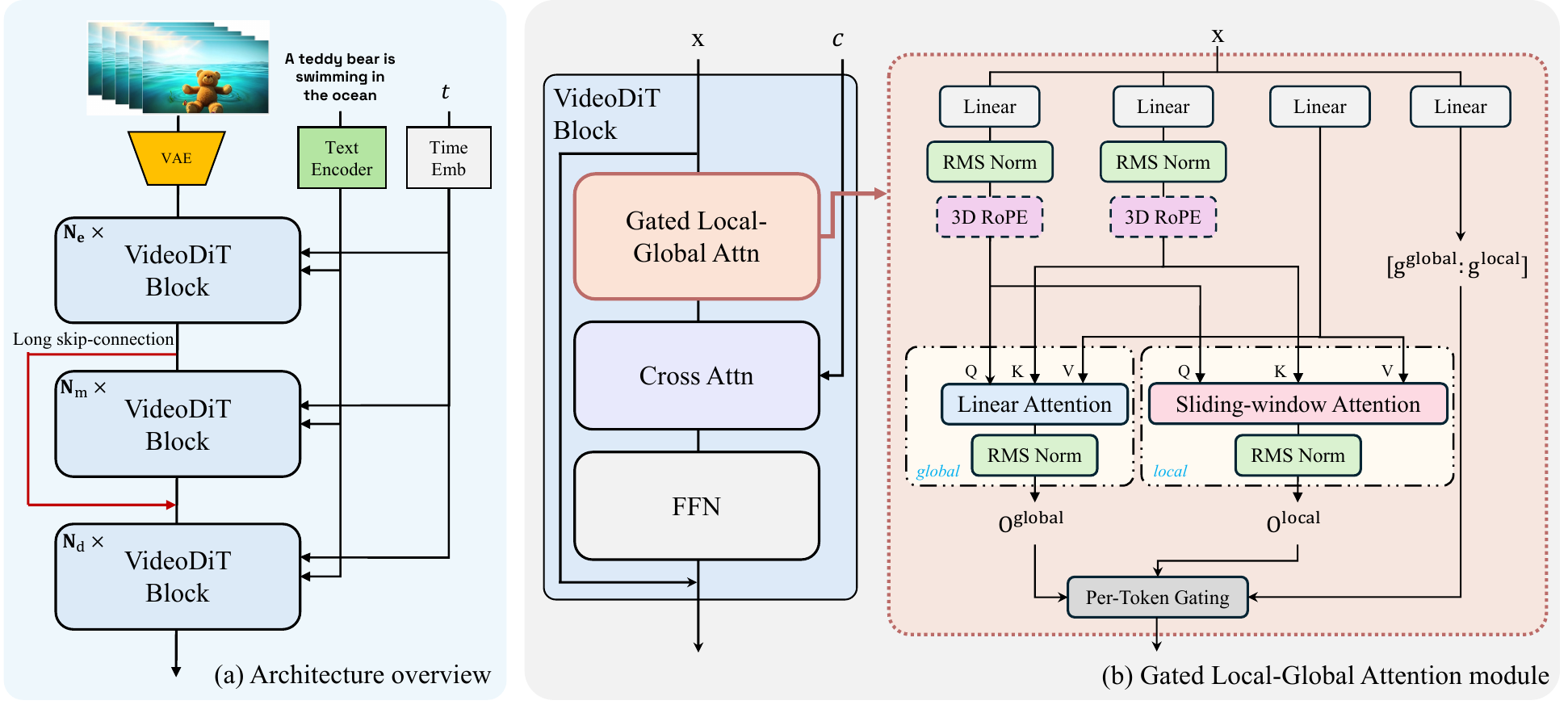}
    \caption{\textbf{Overview of our VideoDiT architecture}. (a)  The network processes video tokens through stacked DiT blocks, utilizing long skip connections to stabilize information flow when tokens are aggressively dropped. (b) Given a DiT block with Gated Local–Global Attention (GLGA) module, input tokens are projected into shared queries, keys, and values. These features are processed by two parallel mechanisms: a linear attention branch for efficient global context and a sliding-window attention branch for expressive local detail. Finally, an input-aware gating ($g$) adaptively fuses the global ($O^{\text{global}}$) and local ($O^{\text{local}}$) outputs on per-token basis.}
    \label{fig:arch}
    \vspace{-1em}
\end{figure}

\section{\modelname: Training Few-Step Video Generator from Scratch}
% \subsection{Overview}
% We propose a novel framework that addresses the two dominant computational bottlenecks of video diffusion transformers: (i) the expensive per-step computation driven by spatiotemporal attention over long token sequences, and (ii) the iterative sampling process, which multiplies this per-step cost by tens of denoising steps. To overcome these, our framework introduces: 1) An \textbf{efficient VideoDiT backbone} featuring \emph{Gated Local--Global Attention (GLGA)}, which substantially reduces per-step compute while remaining robust to aggressive token dropping. 2) An \textbf{efficient few-step VideoDiT training recipe} that reduces the training overhead of solution-flow matching~\cite{luo2025soflow} via path-drop guided training, while incorporating a novel global long-jump consistency objective to minimize integration errors and ensure high-fidelity generation at extremely low step counts.

We propose a novel framework that addresses the two computational bottlenecks of video diffusion training: (i) the per-step spatiotemporal attention computation, and (ii) the number of denoising steps in iterative sampling. To overcome these, our framework introduces: 1) An \textbf{efficient VideoDiT backbone} featuring \emph{Gated Local--Global Attention (GLGA)}, which substantially reduces per-step compute while remaining robust to aggressive token dropping. 2) An \textbf{efficient few-step VideoDiT training recipe} that reduces the training overhead of solution-flow matching~\cite{luo2025soflow} via path-drop guided training, while incorporating a novel global long-jump consistency objective to minimize integration errors and ensure high-fidelity generation at extremely low step counts.

\subsection{Efficient \& Token Droppable VideoDiT Architecture}
\label{sec:droppable_videodit}

%Our architecture targets the dominant per-step bottleneck of video diffusion transformers: the quadratic complexity of spatiotemporal softmax attention. While replacing softmax with linear attention reduces this complexity, prior video DiT designs often attempt to recover lost locality and expressivity by adding convolutional mixing (\eg, Mix-FFN in Sana-Video~\cite{chen2025sana}). However, convolution assumes a consistent spatial neighborhood layout. This makes it fundamentally incompatible with aggressive random token dropping, as dropping tokens destroys the grid locality that convolutions rely upon.
We address the primary bottleneck in VideoDiTs: the quadratic complexity of spatiotemporal attention. While replacing softmax with linear variants reduces this overhead, existing designs typically attempt to recover lost locality and expressivity by incorporating convolutional mixing (\eg, the Mix-FFN in Sana-Video~\cite{chen2025sana}). However, convolutions rely on a consistent spatial neighborhood layout.
It makes them incompatible with aggressive random token dropping, as removing tokens destroys the structured grid locality upon which convolutions depend.
%This makes them fundamentally incompatible with aggressive random token dropping, as removing tokens destroys the structured grid locality upon which convolutions depend.

In contrast, we propose a \textbf{token-droppable hybrid attention} block that preserves linear efficiency and recovers local expressivity through a mechanism that remains well-defined even when tokens are dropped. Our \textbf{Gated Local–Global Attention (GLGA)} block combines:
\textbf{(i)} \emph{Global Linear Attention} for long-range spatiotemporal context, and
\textbf{(ii)} \emph{Local Sliding-Window Attention} for high-frequency details, and \textbf{(iii)} \emph{Input-Aware Gating} to adaptively fuse the two branches at a per-token level. Since sliding-window attention operates over the set of available tokens within a window rather than a dense grid, it is naturally compatible with token dropping. Overall architecture is illustrated in Fig.~\ref{fig:arch}. 

%Crucially, unlike convolutions, sliding-window attention operates over the set of available tokens within a specified radius rather than a fixed coordinate grid. This makes GLGA naturally robust to token dropping, as the local attention mechanism simply attends to the surviving subset of neighbors.
%\\
%GLGA provides linear global aggregation, expressive local refinement, and dynamic routing to balance the two. \\

\noindent\textbf{Global Linear Attention.}
Given an input sequence $X\in\mathbb{R}^{N\times d}$, we compute queries, keys, and values
$Q=XW_Q$, $K=XW_K$, $V=XW_V$, apply Rotary Positional Embeddings (RoPE)~\cite{su2021roformer} to $Q$ and $K$.
For clarity, we describe a single-head formulation; multi-head attention follows standard practice.
% For clarity we describe a single head; multi-head attention applies the same computation per head.
Following~\cite{team2025kimi,yang2024gated}, we use a denominator-free linear attention variant:
% (Gated Delta Networks, kimi linear, :
{
\setlength{\abovedisplayskip}{2pt}
\setlength{\belowdisplayskip}{2pt}
\begin{equation}
    O_{i}^{\text{global}} = Q_i\,S, \qquad
    S = \sum_{j=1}^{N} K_j^{\top} V_j \in \mathbb{R}^{d\times d},
    \label{eq:global_linear_attn}
\end{equation}
}
where $S$ is a global key--value summary shared across all queries. This reduces the quadratic $QK^\top$ interaction to a
single global aggregation plus per-token projection, yielding linear scaling in $N$.
We omit the normalization denominator typical of kernelized linear attention, as it can induce numerical instability when combined with RoPE. Instead, we apply RMS normalization to each branch output, improving training stability and ensuring well-behaved fusion with the local attention branch.
%We omit the normalization denominator typically found in kernelized linear attention, as it can become numerically unstable when combined with RoPE. Instead, we apply RMS normalization to each branch output, improving training stability and yielding well-behaved fusion with the local attention branch. %\\
%Notably, we omit the normalization denominator typically found in kernelized linear attention, as it can induce numerical instability when coupled with RoPE-transformed keys. To ensure stable gradients and balanced feature scales, we instead apply RMS normalization to the output of the global branch. This stabilization is critical for well-behaved fusion with the local attention branch, particularly in the multi-pass training regime of solution-flow.

\noindent{\textbf{Local Window Attention.}}
Since linear attention can underfit local structures~\cite{fan2025breaking,yang2024gated,chen2025sana}, we compute softmax attention over a sliding window based on rasterized token order. Let $\mathcal{W}(i)$ denote the index set of keys within the window of radius $|\mathcal{W}|$ centered at token $i$. We compute
{
\setlength{\abovedisplayskip}{2pt}
\setlength{\belowdisplayskip}{2pt}
\begin{equation}
O_i^{\text{local}} = \sum_{j\in \mathcal{W}(i)} \mathrm{softmax}\left(\frac{Q_i K_j^\top}{\sqrt{d}}\right) V_j.
\label{eq:local_attn}
\end{equation}
}
Note that removing tokens simply reduces the size of $\mathcal{W}(i)$ without altering the operation itself. This requires $\mathcal{O}(N |\mathcal{W}|)$ compute, where $|\mathcal{W}|$ is the maximum tokens per window. To minimize parameter and compute overhead, we reuse the same $Q, K, V$ projections as the linear-attention branch. %\\

%While linear attention achieves global reach, it can underfit high-frequency local structures~\cite{fan2025breaking,yang2024gated,chen2025sana}. To recover this expressivity, we compute softmax attention over a sliding window based on the rasterized token order. Let $\mathcal{W}(i)$ denote the index set of surviving keys within a window of radius $w$ centered at token $i$. We compute:
%Crucially, unlike convolutions that require a fixed grid, removing tokens simply reduces the cardinality of $\mathcal{W}(i)$ without altering the operation. This maintains a complexity of $\mathcal{O}(N w)$, where $w$ is the window size. To minimize parameter overhead, we reuse the $Q, K, V$ projections from the global branch, allowing the local branch to act as a high-resolution refinement of the global context.

\noindent{\textbf{{Input-Aware Gating and Fusion.}}}
To dynamically balance global and local contexts, we predict per-token scalar gates $g^{\text{global}}_i, g^{\text{local}}_i \in [0,1]$:
\begin{equation}
    g =[g^{\text{global}}, g^{\text{local}}] = \sigma(X W_g),
\end{equation}
where $W_g \in \mathbb{R}^{d \times 2}$ is a linear projection and $\sigma$ is the sigmoid function. The final output is the gated sum:
\begin{equation}
    O_i = g^{\text{global}}_i \odot O_i^{\text{global}} + g^{\text{local}}_i \odot O_i^{\text{local}}.
    \label{eq:hybrid_fusion}
\end{equation}

\noindent{\textbf{{Sparse softmax interleaving \& Token-dropping compatibility.}}}
To further boost expressivity, we sparsely interleave full softmax attention blocks (\eg, every $K$ blocks) following recent hybrid architectures~\cite{gemmateam2025gemma3technicalreport,Beltagy2020Longformer}. Because our GLGA backbone relies solely on attention and pointwise layers—avoiding spatial convolutions—it is inherently compatible with random token dropping. 
%Following the SPRINT~\cite{park2025sprint} architecture, we divide our network into a dense encoder, a sparse middle stack, and a dense decoder. We apply token dropping exclusively to the middle stack, and integrate a long skip connection bridging the encoder and decoder to stabilize information flow across the heavily dropped intermediate layers. \\
%To further enhance expressivity, we sparsely interleave full softmax attention blocks (\eg, every $K$-th layer) following recent hybrid transformer paradigms~\cite{gemmateam2025gemma3technicalreport,Beltagy2020Longformer}. Because our GLGA backbone relies exclusively on attention and pointwise operations—eschewing grid-dependent spatial convolutions—the entire architecture remains inherently robust to aggressive token dropping. This design ensures that both local and global dependencies are preserved regardless of the sparsity pattern.

\begin{figure}[t]
    \centering
    \includegraphics[width=\linewidth]{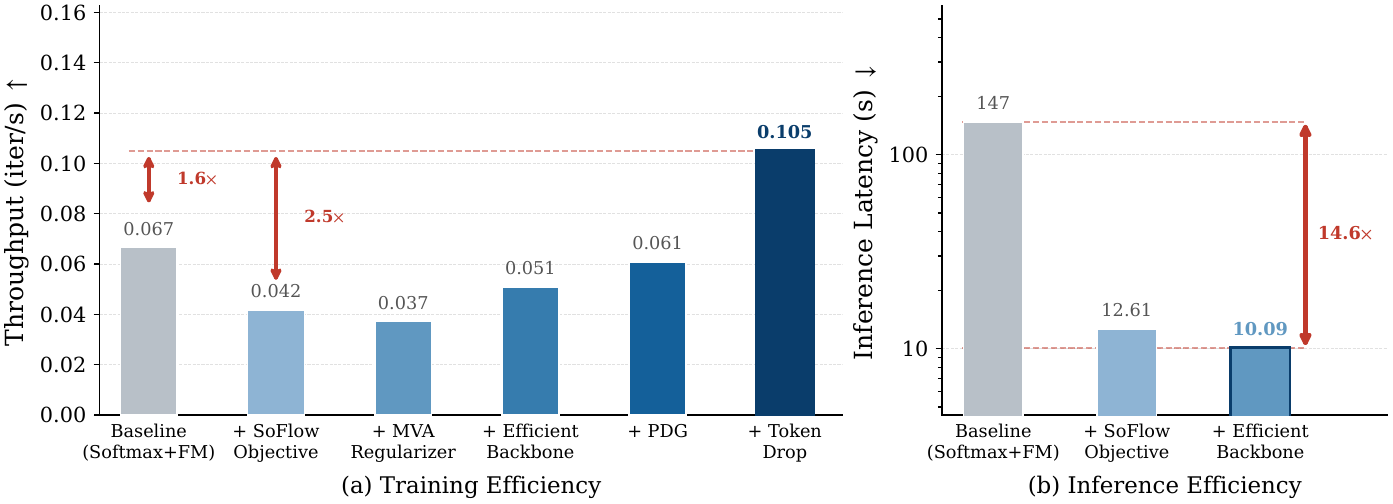}
    \caption{\textbf{End-to-end efficiency gains from our framework.} (a) \textbf{Training efficiency} measured as throughput (iterations/sec). Relative to the standard flow-matching baseline, introducing the solution-flow (SoFlow) objective and the MVA regularizer increases per-iteration overhead and therefore reduces throughput. In contrast, our efficient backbone, PDG training, and token dropping provide consistent speedups, improving throughput by up to $1.6\times$ and $2.5\times$ over the baselines. (b) \textbf{Inference efficiency.} Our overall framework substantially reduces end-to-end inference time relative to the softmax+flow-matching baseline ($147\rightarrow10.09$s).}
    \label{fig:efficiency}
    \vspace{-1.5em}
\end{figure}

\subsection{Efficient Few-Step VideoDiT Training}
\label{sec:pdg}
We build on SoFlow~\cite{luo2025soflow} to train a bi-time solution function $f_\theta(x_t,t,s,c)$ for few-step generation.
While this enables teacher-free few-step training from scratch, scaling guided solution-flow training to long,
high-resolution videos is challenging due to its substantially higher training cost than standard flow matching. Pseudo-codes for training and inference are provided in Algorithms~\ref{alg:training} and~\ref{alg:inference}.

\smallskip
\noindent{\textbf{Computational overheads in video solution-flow training.}}
Compared to standard flow matching, solution-flow training introduces additional forward-pass overhead from two sources.
First, the \textit{Solution-Consistency (SCM)} objective evaluates the model at multiple time pairs, requiring extra model evaluations beyond the single forward used in flow matching.
Second, \textit{Classifier-Free Guidance (CFG)} distillation necessitates constructing a guided velocity target during training:
\begin{equation}
v_g(x_t,t,c\,|\,x_0) \;=\; w\,v(x_t,t,c\,|\,x_0) + (1-w)\,v(x_t,t,\emptyset\,|\,x_0),
\end{equation}
In SoFlow, the conditional term is approximated by the ground-truth velocity $v(x_t,t,c|x_0)$, but the unconditional term requires an additional network evaluation $v_\theta(x_t,t,\emptyset)$. Combined with the heavy quadratic attention of video DiTs, these extra passes make naive solution-flow training prohibitively expensive to scale. 

\smallskip
\noindent\textbf{Scaling Solution-Flow via Multi-Level Efficiency.}
% We address these costs with complementary mechanisms that reduce the cost of each forward and reduce guidance overhead.
% (i) We use our \emph{efficient GLGA backbone} (\cref{sec:droppable_videodit}) to lower per-forward attention cost.
% (ii) We apply \emph{token dropping} during most of training to shorten the sequence length, reducing compute and memory for every forward pass (including SCM).
% (iii) To slash guidance overhead, we propose \emph{Path-Drop Guided (PDG) training}, which replaces the unconditional forward with a cheap \emph{weak-path} forward that skips a large fraction of middle transformer blocks. Together, these changes make guided solution-flow training practical at large-scale video. %\\
%We mitigate these overheads through several complementary mechanisms. First, our efficient GLGA backbone (\cref{sec:droppable_videodit}) lowers per-pass attention complexity. Next, token dropping shortens the sequence length during training, reducing the compute and memory footprint of every forward pass. To slash guidance overhead, Path-Drop Guided (PDG) training replaces the unconditional pass with a cheap weak-path forward that skips middle transformer blocks. Finally, Mean-Velocity Additivity (MVA) reduces integration error during large-jump few-step sampling. Together, these innovations make guided solution-flow practical for large-scale video.
% We address these costs with complementary mechanisms that reduce the cost of each forward and reduce guidance overhead.
We propose complementary mechanisms to reduce the cost of each forward-pass and reduce guidance overhead.
First, we use our \emph{efficient GLGA backbone} (\cref{sec:droppable_videodit}) to lower per-forward attention cost. Next, we apply \emph{token dropping} during most of training to shorten the sequence length, reducing compute and memory for every forward pass (including SCM). Finally, to slash guidance overhead, we propose \emph{Path-Drop Guided (PDG) training}, which replaces the unconditional forward with a cheap \emph{weak-path} forward that skips a large fraction of middle transformer blocks. Additionally, we introduce \emph{Mean-Velocity Additivity (MVA)} to reduce error accumulation in solution-flow over large-jumps in few-step sampling. Together, these changes make guided solution-flow training practical at large-scale video. 

\smallskip
\noindent{\emph{(a) Path-Drop Guided (PDG) training.}}
Given a model with $L$ blocks $\{B_1,\dots,B_L\}$, we define a weak path that skips intermediate blocks $B_{i+1}$ through $B_j$ as:
\begin{equation}
\mathrm{Weak}(X) = (B_L \circ \cdots \circ B_{j+1}) \circ \mathrm{Id} \circ (B_{i} \circ \cdots \circ B_{1})(X),
\qquad 1 \le i < j \le L,
\label{eq:weak_path}
\end{equation}
where $\mathrm{Id}$ denotes the identity mapping. Conveniently, we align this skipped region $(i, j)$ with the sparse middle stack of our architecture.
We replace the expensive unconditional baseline in CFG with this weak-path prediction:
\begin{equation}
v_{\text{PDG}}(x_t,t,c\,|\,x_0) = w\,v(x_t, t, c\,|\,x_0)+(1-w)\,v_{\theta}^{\text{weak}}(x_t,t,\emptyset).
\label{eq:pdg_cfg}
\end{equation}
We then plug $v_{\text{PDG}}$ into SoFlow’s FM and SCM objectives to train the full-path model. Inspired by early works~\cite{karras2024guiding,park2025sprint,haji2026one}, interpolating with a computationally "free" weak model provides a sufficiently aligned baseline to stabilize guided targets while eliminating the compute overhead of a full unconditional pass. 

\smallskip
\noindent{\emph{(b) Global long-jump consistency via Mean-Velocity Additivity.}}
Local consistency objectives in Eq.~\ref{eq:soflow_scm} accumulate integration errors over the large jumps required for few-step sampling. To reduce this, we introduce a global regularizer derived from the exact additivity of the mean velocity. Let $\bar v_\theta(x_t,t,s,c) \triangleq \frac{x_t - f_\theta(x_t,t,s,c)}{t-s}$ denote the predicted mean velocity over $[s,t]$. For the ground-truth flow, displacements are additive across any intermediate time $s < \ell < t$, yielding the exact constraint:
\begin{equation*}
    (t-s)\;\bar v(x_t,t,s,c) = (t-\ell)\;\bar v(x_t,t,\ell,c) + (\ell-s)\;\bar v(x_\ell,\ell,s,c).
% \label{eq:mva}
\end{equation*}
We enforce this constraint using a stop-gradient network, $\theta^{-}$. By substituting the mean velocity definition into the $L_2$ residual of this constraint, the $x_t$ terms perfectly cancel. The mean-velocity additivity (MVA) objective simplifies to a direct composition error, enforcing the semigroup property of the flow:
\begin{equation}
\mathcal{L}_{\mathrm{MVA}} = \mathbb{E}_{x_t,c,t,\ell,s} \left[ \left\| f_\theta(x_t,t,s,c) - f_{\theta^{-}}(f_{\theta^{-}}(x_t,t,\ell,c),\ell,s,c) \right\|_2^2 \right].
\label{eq:mva}
\end{equation}
% Crucially, by our theoretical analysis (detailed in Appendix), minimizing MVA loss strictly controls the global integration error over large discretization steps, providing a theoretical guarantee for long-jump consistency. 
Our theoretical analysis (see Appendix~\ref{app:mva_theory}) shows that minimizing MVA loss strictly controls the global integration error over large discretization steps. 
To directly align the training objective with the inference trajectory, we restrict $t$, $s$, and $\ell$ to a predefined discrete set of evaluation timesteps (\eg, an 64-step schedule), and randomly sample long-jump intervals $(t,s,l)$ from this set such that $s < \ell < t$. 
We optimize the total loss via a compute-efficient batch-splitting strategy: $\mathcal{L} = a \mathcal{L}_{\mathrm{FM}} + b \mathcal{L}_{\mathrm{SCM}} + (1-a-b) \mathcal{L}_{\mathrm{MVA}}.$ 
Specifically, we partition each batch into three disjoint groups with ratios $a$, $b$, and $(1-a-b)$. 
Thus, the two additional EMA forward passes required for $\mathcal{L}_{\mathrm{MVA}}$ are only computed on a strict subset of the batch. 
% We further mitigate this cost by applying aggressive token dropping via our GLGA backbone during the MVA target evaluations. 
We further lower this cost via aggressive token dropping in our GLGA backbone. 
To ensure stable convergence, we employ a delayed-activation schedule: the MVA ratio $(1-a-b)$ is strictly zero during early training, allowing the model to learn the base flow dynamics before the long-jump consistency constraint is introduced. 

\smallskip
\noindent{\textbf{Empirical Efficiency Gains.}}
Fig.~\ref{fig:efficiency} summarizes our cumulative gains on 480p video.
On the training side (throughput in iters/s), adopting SoFlow incurs 40\% additional overhead compared to a Flow-Matching (0.067 $\rightarrow$ 0.042 it/s) due to extra forward passes from solution consistency and guidance. Integrating our global MVA regularizer adds a 7\% further cost (0.042 $\rightarrow$ 0.037 it/s). Our efficient GLGA backbone begins to recover this compute cost (0.037 $\rightarrow$ 0.051 it/s), and PDG further minimizes guidance overhead (0.051 $\rightarrow$ 0.061 it/s). Finally, enabling aggressive token dropping yields the largest gain (0.061 $\rightarrow$ 0.105 it/s), resulting in an overall \textbf{1.6$\times$ and 2.5$\times$ training speedup} over standard FM and SFM, respectively.
On the inference side (latency), few-step solution-flow sampling provides a massive reduction versus standard iterative diffusion (147 $\rightarrow$ 12.61 s), and our efficient backbone further reduces the per-step cost (12.61 $\rightarrow$ 10.09 s). This culminates in a \textbf{14.6$\times$ overall inference latency reduction}. 
\section{Experiment}

\subsection{Implementation Details}
Our architecture integrates the Wan~\cite{wan2025wan} structure with SPRINT~\cite{park2025sprint} sparse-training paradigm. 
Specifically, our transformer blocks follow an encoder--decoder structure connected by a long residual path (\cref{fig:arch}), with each block comprising self-attention, cross-attention, and a feed-forward network. We replace self-attention with our GLGA, sparsely interleaving full softmax attention every 8 layers.
We conduct experiments on two datasets: (i) class-conditional Kinetics-700~\cite{kay2017kinetics}, and (ii) a large-scale, proprietary text-to-video dataset. All models are trained from scratch with random initialization. We apply 75\% token-dropping to the middle stack for the first 80\% of training and fine-tune the remainder with full tokens. 
Detailed model hyperparameters and training configurations are provided in the Appendix~\ref{app:hyper}.

\begin{figure}[t]
  \centering
  \begin{minipage}[t]{0.49\linewidth}
    \centering
    \includegraphics[width=\linewidth]{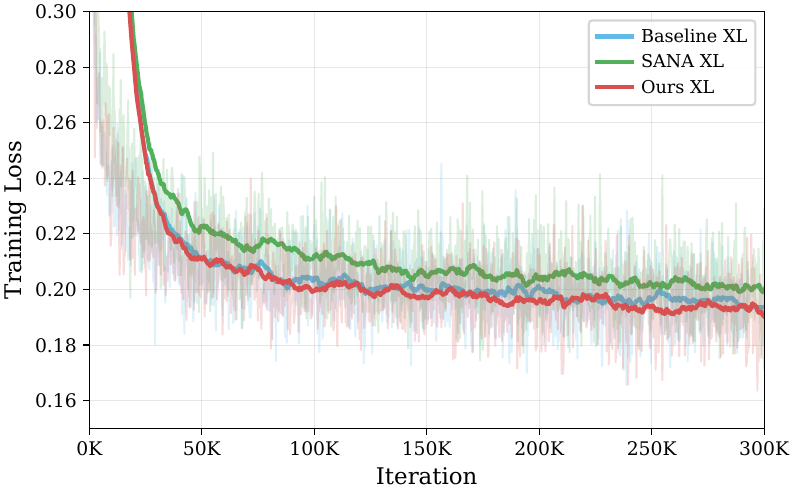}
    \vspace{-0.3em}
    {\small (a) Training loss vs. iterations.}
  \end{minipage}\hfill
  \begin{minipage}[t]{0.49\linewidth}
    \centering
    \includegraphics[width=\linewidth]{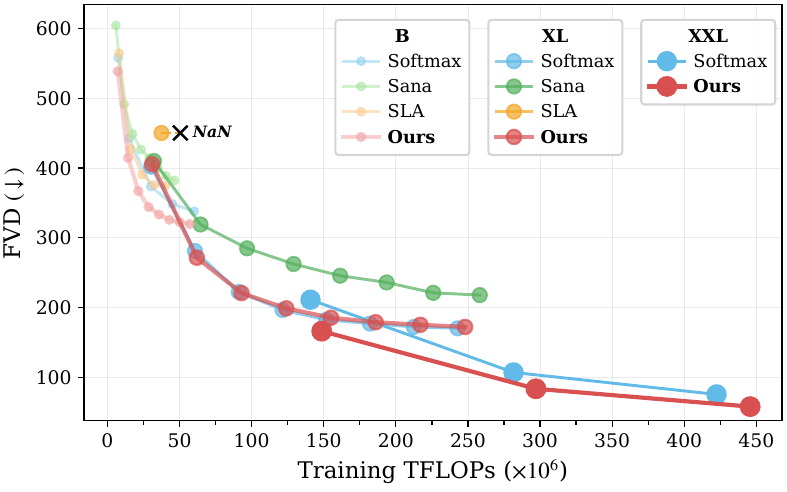}
    \vspace{-0.3em}
    {\small (b) FVD vs. training FLOPs.}
  \end{minipage}
  \caption{\textbf{Systemic comparison on Kinetics 700.}
(a) \textbf{Training loss vs.\ iterations} comparing the softmax-attention baseline, Sana-Video attention, and our GLGA module. 
(b) \textbf{FVD vs.\ training compute} across various model scales for softmax, Sana-Video attention, sparse-linear attention (SLA), and GLGA. GLGA consistently achieves lower loss and lower FVD at matched compute, and its advantage persists as model scale increases; SLA become
unstable and diverge under the same training setup.}
  \label{fig:performance_comparison}
  \vspace{-1em}
\end{figure}

\subsection{Controlled Experiments on Kinetics-700}
To isolate the expressiveness of our efficient self-attention block, we systematically compare GLGA against several baselines using an identical model structure (without token-dropping). These include standard softmax attention, linear attention with Mix-FFN used in Sana-Video~\cite{chen2025sana}, and Sparse Linear Attention (SLA)~\cite{zhang2025sla}. We set SLA to an 80\% sparsity level (default=95\%), as we empirically observed that higher ratios induced training instability.

\smallskip
\noindent{\textbf{Expressiveness.}}
Fig.~\ref{fig:performance_comparison}(a) reports the training loss over iterations. GLGA matches or outperforms the softmax baseline in early optimization and attains a consistently lower loss throughout the training. In contrast, Sana-Video's attention exhibits substantially higher loss across the full schedule, indicating slower convergence and a higher final error. This shows that GLGA’s hybrid design—with its adaptive gating between global aggregation and local refinement—preserves optimization stability while retaining superior modeling capacity.
Fig.~\ref{fig:performance_comparison}(b) highlights scalability by plotting FVD~\cite{unterthiner2019accurategenerativemodelsvideo} against training FLOPs across three model scales (B, XL, XXL). Across all compute budgets and parameter scales, GLGA consistently achieves a lower FVD at matched compute, effectively delivering better generation quality per FLOP. Notably, this advantage persists as we scale from B $\rightarrow$ XL $\rightarrow$ XXL, with GLGA scaling much more favorably than the sparse/linear baselines and remaining highly competitive with the heavy softmax baseline. Finally, we note that even at the lowered 80\% sparsity level, SLA proved highly unstable and diverged to NaN across all scaled regimes, severely limiting training ability. In contrast, GLGA remains remarkably stable across all scales while delivering an optimal quality--efficiency trade-off.

\smallskip
\noindent{\textbf{Improvement against baselines.}}
Tab.~\ref{tab:fvd_comparison} disentangles our architectural and objective improvements against baselines. 1) \textbf{Base1$\rightarrow$Base2}: Swapping softmax for Sana's attention block degrades FVD from 188.9 to 217.7, demonstrating limited capacity of Sana. 2) \textbf{Base1$\rightarrow$Ours2}: GLGA backbone trained with flow-matching fully recovers the performance (189.7 FVD). 3) \textbf{Ours2$\rightarrow$Ours4}: Upgrading to our PDG-SFM objective then yields a massive compound benefit: reducing inference cost by 10$\times$ (40 to 4 NFEs) while simultaneously improving generation quality to 146.7 FVD. Qualitative result in Fig.~\ref{fig:qual_figure} further demonstrates this substantial improvement.

\subsection{Comparison of Text-to-Video Models}
We evaluate our 1.5B \modelname model on the challenging task of high-resolution (480p) text-to-video generation. To demonstrate its competitive performance and efficiency, we compare against recent state-of-the-art video models, including Sana-Video~\cite{chen2025sana}, rCM~\cite{zheng2025large}, SLA~\cite{zhang2025sla}, VSA~\cite{zhang2025vsa}, and Wan 2.1~\cite{wan2025wan}. 

\smallskip
\noindent{\textbf{Efficiency.}}
Fig.~\ref{fig:efficiency_comparison}(a) reports forward-pass latency against token count. Wan 2.1 with softmax attention scales poorly, with latency increasing sharply in the high-token regime. While SLA reduces cost relative to softmax, it remains noticeably slower as the token count grows for sparsity of 90\%.
Sana-Video is slightly faster than GLGA at 125K tokens, reflecting its streamlined linear-style design. However, the experiments (in Fig.~\ref{fig:performance_comparison} and Tab.~\ref{tab:fvd_comparison}) show that this speed advantage comes with a degradation in optimization and quality: Sana-Video attains higher training loss and worse FVD at matched training compute.
In contrast, \modelname strikes a superior balance, outperforming softmax in speed while strictly preserving generation quality.
%running significantly faster than softmax while strictly preserving quality. 
Furthermore, \modelname is compatible with token dropping—applying a 75\% token drop yields the absolute lowest latency and the most favorable scaling curve, effectively handling the high-resolution compute bottleneck.
Beyond per-step scaling, Fig.~\ref{fig:efficiency_comparison}(b) demonstrates the compounding benefits of our framework for end-to-end 480p video generation. By our efficient architecture with the 4-step PDG training recipe, \modelname (1.5B) achieves a remarkable 3.09s diffusion latency (excluding VAE). This is $1.4\times$ to $1.8\times$ faster than recent few-step distillation methods like VSA and rCM, and completely eclipses standard 50-step iterative models—running $23.3\times$ faster than Sana (2.0B), $45.3\times$ faster than Wan2.1-1.3B, and up to $234.6\times$ faster than the heavy Wan2.1-14B.

\noindent{\textbf{Quantitative result.}}
To provide a comprehensive assessment of generation quality across diverse semantic and temporal dimensions, detailed VBench~\cite{huang2023vbench} evaluations comparing \modelname against the baselines are provided in the Tab.~\ref{tab:vbench}.

\begin{figure}[t]
    \centering
    \includegraphics[width=\linewidth]{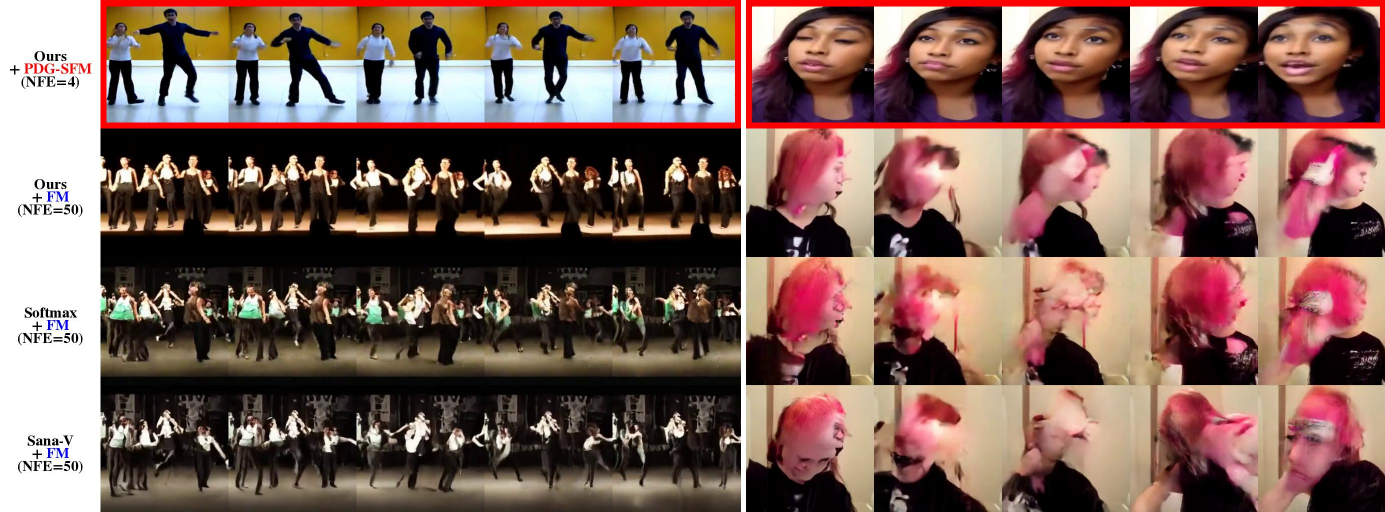}
    \caption{\textbf{Qualitative comparison on Kinetics-700} on 400K training iterations. Our complete framework (Ours+PDG-SFM) successfully generates highly detailed, temporally consistent videos in only 4 inference steps. In contrast, baselines trained with standard Flow Matching (FM)—including standard softmax, and Sana-video's attention—struggle to resolve vivid details and exhibit severe blurring or structural distortion despite using 50 steps.}
    \label{fig:qual_figure}
    \vspace{-1em}
\end{figure}

\vspace{-2.5mm}
\subsection{Analysis}
\vspace{-1mm}

Here, we present results of the analysis conducted on Kinetics-700 dataset. 

\noindent{\textbf{Ablation on architecture.}} 
Tab.~\ref{tab:attn_gating_ablation} isolates GLGA's internal components. Relying solely on global linear or local sliding-window attention yields poor quality (420.8 and 390.2 FVD, respectively). While element-wise summation improves FVD to 352.9, our input-aware gating achieves the best result (325.7 FVD), improving 22\% from pure linear attention. This confirms that dynamically routing global context and local detail per-token is essential for high-fidelity generation. 

\noindent{\textbf{Ablation on training objective.}} 
Tab.~\ref{tab:training_objective} evaluates objectives using a fixed architecture. Standard SFM reduces sampling to 4 NFEs, (347.1 FVD) but incurs massive multi-pass training overhead as described in Fig.~\ref{fig:efficiency}(a). Our PDG-SFM completely recovers training throughput driving the 4-step FVD down to 322.2 via weak-path guidance. 

\noindent{\textbf{Ablation on Mean-Velocity Additivity.}}
In Tab.~\ref{tab:fvd_comparison}, \textbf{Ours3$\rightarrow$Ours4} isolates the impact of MVA regularizer. While the PDG-SFM objective alone can reduce sampling to 4 NFEs (185.2 FVD), it still suffers from integration errors across large time jumps. Explicitly enforcing long-jump consistency via MVA corrects these errors, driving the FVD down to 146.7 at the exact same 4-step inference cost. This substantial 38.5-point improvement confirms that MVA is essential for preserving fidelity in few-step solution flows. 

\begin{table}[t]
\centering
\caption{\textbf{FVD comparison} against baselines. All models are trained for 400K iterations.}
\setlength{\tabcolsep}{8pt}
\renewcommand{\arraystretch}{1.15}
\resizebox{\textwidth}{!}{%
\begin{tabular}{l|c|cc|c|cc|c}
\hline % Replaced \toprule
 & \textbf{Attn.} & \textbf{Objective} & \textbf{MVA} & \textbf{Token drop} & \textbf{Params} & \textbf{NFE} $\downarrow$ & \textbf{FVD} $\downarrow$ \\
\hline % Replaced \midrule
\hline
Base1 & Softmax & FM      & \textcolor{Red}{\ding{55}} & \textcolor{Red}{\ding{55}} & 457M & 40 & 188.9 \\
Base2 & Sana    & FM      & \textcolor{Red}{\ding{55}} & \textcolor{Red}{\ding{55}} & 607M & 40 & 217.7 \\
Base3 & Sana    & FM      & \textcolor{Red}{\ding{55}} & \ding{51} & 607M & 40 & 456.1 \\
\hline % Replaced \midrule
\textbf{Ours1} & GLGA & FM      & \textcolor{Red}{\ding{55}} & \textcolor{Red}{\ding{55}} & 529M & 40 & 191.3 \\
\textbf{Ours2} & GLGA & FM      & \textcolor{Red}{\ding{55}} & \ding{51} & 529M & 40 & 189.7 \\
\textbf{Ours3} & GLGA & PDG-SFM  & \textcolor{Red}{\ding{55}} & \ding{51} & 529M & \textbf{4} & 185.2 \\
\rowcolor{gray!12}
\textbf{Ours4} & GLGA & PDG-SFM  & \ding{51} & \ding{51} & 529M & \textbf{4} & \textbf{146.7} \\
\hline % Replaced \bottomrule
\end{tabular}%
} % end resizebox
\label{tab:fvd_comparison}
\end{table}

\begin{table}[t]
\centering
\small
\setlength{\tabcolsep}{6pt}
\renewcommand{\arraystretch}{1.15}

\begin{minipage}[t]{0.63\linewidth}
\centering
\captionof{table}{\textbf{Ablation study on our GLGA components}. All trained for 300K iterations.}
\label{tab:attn_gating_ablation}
\resizebox{\linewidth}{!}{%
\begin{tabular}{lccc|c}
\hline
 & \textbf{Linear Attn.} &
 \textbf{\makecell{Sliding Window\\Attn.}} &
 \textbf{\makecell{Input-aware\\Gating}} &
 \textbf{FVD} $\downarrow$ \\
\hline
\hline
(a) & \cmark &  &  & 420.8 \\
(b) &  & \cmark &  & 390.2 \\
(c) & \cmark & \cmark &  & 352.9 \\
\rowcolor{gray!15}
(d) & \cmark & \cmark & \cmark & \textbf{325.7} \\
\hline
\end{tabular}%
}
\end{minipage}\hfill
\begin{minipage}[t]{0.36\linewidth}
\centering
\captionof{table}{\textbf{Ablation on training objectives}. All trained for 300K iterations.}
\label{tab:training_objective}
\resizebox{\linewidth}{!}{%
\begin{tabular}{lc|c}
\hline
\textbf{Method} & \textbf{NFE} $\downarrow$ & \textbf{FVD} $\downarrow$ \\
\hline
\hline
FM & 50 & 325.7 \\
FM & 4 & 513.8 \\
\hline
SFM & 4  &  347.1\\
\rowcolor{gray!15}
\textbf{PDG-SFM (Ours)} & 4  &  \textbf{322.2} \\
\hline
\end{tabular}%
}
\end{minipage}

\end{table}

\noindent{\textbf{Robustness to token-dropping.}}
Tab.~\ref{tab:fvd_comparison} also validates GLGA's compatibility with token-drop training. \textbf{Base2$\rightarrow$Base3} shows that dropping tokens in the Sana baseline causes a catastrophic collapse in FVD (456.1 from 217.7) because its convolutions assume a dense spatial grid. In contrast, \textbf{Ours1$\rightarrow$Ours2} demonstrates that GLGA's local branch relies on sliding-window attention—which operates naturally over unstructured sets—allowing it to survive severe dropping and unlock massive training speedups without degradation.

\section{Conclusion}
In this paper, we introduced \modelname, an efficient framework that tackles the dual bottlenecks of softmax-attention complexity and iterative sampling in video diffusion models. We proposed Gated Local--Global Attention, a token-droppable attention that maintains expressivity and efficiency under aggressive token dropping. We also developed Path-Drop Guided solution-flow matching, augmented by a Mean-Velocity Additivity regularizer to ensure efficient and high-fidelity generation. Ultimately, \modelname enables practical few-step video generator training from scratch, establishing a highly efficient paradigm for future video synthesis.

\clearpage

\clearpage  % TODO FINAL: This \clearpage needs to be removed from both review and camera-ready versions.

% \section*{Acknowledgements}
% Please insert your acknowledgments here.

\bibliography{main}

\clearpage

\appendix

% --- Manual Table of Contents ---
% \section*{Appendix Contents}
% \etocsettocstyle{}{}
% \localtableofcontents
% \clearpage

\section{Theoretical Analysis of Mean-Velocity Additivity}
\label{app:mva_theory}

\subsection{Comparison to Local Solution Consistency Bounds}
\label{app:compare_local_bounds}
Solution Flow (SoFlow)~\cite{luo2025soflow} provides a principled analysis of few-step generation by relating the \emph{global} solution-function error to a \emph{local} partial differential equation (PDE) residual. Concretely, SoFlow shows that the discrepancy between the learned solution map $f_\theta(x_t,t,s,c)$ and the ground-truth flow map $f(x_t,t,s,c)$ can be expressed as an integral of a residual term along the true trajectory, implying a bound of the form:
\begin{equation}
\|f(x_t,t,s,c) - f_\theta(x_t,t,s,c)\|_2 \;\le\; |t-s|\,\delta,
\label{eq:compare_soflow_global}
\end{equation}
whenever the PDE residual is uniformly bounded by $\delta$. To obtain an explicit $\delta$, SoFlow connects its local solution-consistency objective to the PDE residual through a Taylor expansion around $t$ with an intermediate point $\ell$ close to $t$, yielding a $\delta$ that scales linearly with $|t-\ell|$.

\smallskip
\noindent
\textbf{Implication for few-step generation.} While the above reasoning is mathematically well-justified in the continuous limit $\ell\!\to\! t$, it highlights a practical challenge for \emph{few-step} sampling, where the jump length $\Delta \triangleq t-s$ becomes macroscopic (e.g., $\Delta \approx 1$ for 1-step sampling). In SoFlow, $\ell$ is determined by a schedule of the form $\ell = t + (s-t)r$, so $|t-\ell| = |t-s|r = \Delta r$, where $r \in [r_{\text{min}},1]$ is an iteration-aware scheduling factor and $r_{\text{min}}$ is small value (\eg, $1\times 10^{-3}$). 
Consequently, the global bound in \eqref{eq:compare_soflow_global} contains a contribution of the form:
$$|t-s| \cdot O(|t-\ell|) \;=\; O(\Delta)\cdot O(\Delta r) \;=\; O(\Delta^2) \cdot r,$$
which can be non-negligible when $\Delta$ is large, severely limiting its accuracy for few-step generation. 
%Intuitively, local finite-difference consistency ensures the learned map behaves correctly under small time changes, but it does not directly constrain whether the learned \emph{finite} jump $(t\!\to\! s)$ remains consistent with the model's own multi-jump trajectory.

\smallskip
\noindent
\textbf{Our perspective: semigroup defect as a global error source.}
In a continuous dynamical system, the exact solution function (\ie, the flow map) satisfies the \emph{semigroup property}: for any $s<\ell<t$, taking one direct step from $t$ to $s$ must perfectly equal stopping at an intermediate time $\ell$ and then stepping to $s$:
\begin{equation}
f(x_t,t,s,c) \;=\; f\!\big(f(x_t,t,\ell,c),\ell,s,c\big).
\label{eq:compare_true_semigroup}
\end{equation}
Neural solution maps only approximate this structure and may violate it over large time gaps. We quantify this violation via the \emph{semigroup defect}—the discrepancy between the model's single-jump prediction $f_\theta(x_t,t,s,c)$ and its own composed multi-jump trajectory $f_\theta\!\big(f_\theta(x_t,t,\ell,c),\ell,s,c\big)$.

In the theoretical analysis below, we establish that our proposed Mean-Velocity Additivity (MVA) objective directly penalizes this semigroup defect. Specifically, Lemma~\ref{lem:mva_semigroup_equiv} proves that the MVA residual is \emph{exactly} the semigroup defect. Building upon this, Theorem~\ref{thm:long_jump_error_bound} and Corollary~\ref{cor:quadratic_shrink} demonstrate that enforcing semigroup consistency suppresses the quadratic contribution of the long-jump error bound from $O(\Delta^2)$ down to $O(\Delta^2/K)$.

\smallskip
\noindent
\textbf{Takeaway.}
Local objectives (e.g., SCM) and global semigroup objectives (MVA) are complementary: local constraints shape the instantaneous dynamics and improve short-step fidelity, while MVA explicitly regularizes the long-interval compositions that are most critical for few-step sampling.

\subsection{Mean-Velocity Additivity is equivalent to Semigroup Consistency}
\label{app:mva_equiv}

We recall the bi-time solution-function parameterization $f_\theta(x_t,t,s,c)$ and define the predicted mean velocity over an interval $[s,t]$ as $\bar v_\theta(x_t,t,s,c) \triangleq \frac{x_t - f_\theta(x_t,t,s,c)}{t-s}$ for $0\le s < t \le 1$.

\begin{lemma}[MVA residual equals semigroup defect]
\label{lem:mva_semigroup_equiv}
Let $g(x_t,t,s,c)$ be any mapping that satisfies the boundary condition $g(x_t,t,t,c)=x_t$.
Fix $0\le s < \ell < t \le 1$ and define $x_\ell \triangleq g(x_t,t,\ell,c)$.
Then the mean-velocity additivity residual in the main paper
\begin{equation}
r_{\mathrm{MVA}}
\;\triangleq\;
(t-s)\bar v_g(x_t,t,s,c)
\;-\;
(t-\ell)\bar v_g(x_t,t,\ell,c)
\;-\;
(\ell-s)\bar v_g(x_\ell,\ell,s,c)
\label{eq:app_mva_residual_def}
\end{equation}
simplifies exactly to the semigroup defect:
\begin{equation}
r_{\mathrm{MVA}}
\;=\;
g(x_\ell,\ell,s,c) - g(x_t,t,s,c).
\label{eq:app_mva_residual_semigroup}
\end{equation}
Consequently, substituting $x_\ell$ back into the norm yields:
\begin{equation}
\|r_{\mathrm{MVA}}\|_2^2
\;=\;
\big\|g\!\big(g(x_t,t,\ell,c),\ell,s,c\big) - g(x_t,t,s,c)\big\|_2^2.
\label{eq:app_mva_equiv_norm}
\end{equation}
\end{lemma}

\begin{proof}
By the definition of mean velocity and $x_\ell$, we have:
\begin{align*}
(t-s)\bar v_g(x_t,t,s,c) &= x_t - g(x_t,t,s,c), \\
(t-\ell)\bar v_g(x_t,t,\ell,c) &= x_t - g(x_t,t,\ell,c) = x_t - x_\ell, \\
(\ell-s)\bar v_g(x_\ell,\ell,s,c) &= x_\ell - g(x_\ell,\ell,s,c).
\end{align*}
Substituting these into Eq.~\eqref{eq:app_mva_residual_def} yields:
\begin{align*}
r_{\mathrm{MVA}}
&=
\big(x_t - g(x_t,t,s,c)\big)
-
\big(x_t - x_\ell\big)
-
\big(x_\ell - g(x_\ell,\ell,s,c)\big) \\
&=
g(x_\ell,\ell,s,c) - g(x_t,t,s,c),
\end{align*}
which proves \eqref{eq:app_mva_residual_semigroup} and hence \eqref{eq:app_mva_equiv_norm}.
\end{proof}

\noindent\textbf{Remark.}
Lemma~\ref{lem:mva_semigroup_equiv} shows that Mean-Velocity Additivity (MVA) is not merely a heuristic regularizer on ``velocities'': the MVA residual reduces \emph{exactly} to a composition (semigroup) defect. Concretely, the residual $r_{\mathrm{MVA}}$ measures the discrepancy between (i) a \emph{direct} jump from $t$ to $s$, $g(x_t,t,s,c)$, and (ii) the model's own \emph{two-leg} trajectory that first jumps to an intermediate time $\ell$ and then to $s$, i.e., $g\!\big(g(x_t,t,\ell,c),\ell,s,c\big)$. Therefore, minimizing $\|r_{\mathrm{MVA}}\|_2^2$ is precisely forcing the learned solution function $f_\theta$ to satisfy the exact semigroup composition property.

\subsection{Long-jump Error Bound under Approximate Semigroup Consistency}
\label{app:long_jump_bound}

Let $f(\cdot,t,s,c)$ denote the ground-truth solution function of the velocity ODE, which satisfies
the \emph{exact} semigroup property:
\begin{equation}
f(x_t,t,s,c) = f\big(f(x_t,t,\ell,c),\ell,s,c\big),\qquad \forall\, s\le \ell \le t.
\label{eq:app_true_semigroup}
\end{equation}
Define the pointwise approximation error
\begin{equation}
e_\theta(x_t;t,s,c) \;\triangleq\; f_\theta(x_t,t,s,c) - f(x_t,t,s,c),
\label{eq:app_pointwise_error}
\end{equation}
and the semigroup defect
\begin{equation}
d_\theta(x_t;t,\ell,s,c)
\;\triangleq\;
f_\theta(x_t,t,s,c)
-
f_\theta\!\big(f_\theta(x_t,t,\ell,c),\ell,s,c\big).
\label{eq:app_semigroup_defect}
\end{equation}

\begin{theorem}[Long-jump error decomposes into short-jump errors + semigroup defect]
\label{thm:long_jump_error_bound}
Assume the following:

\smallskip
\noindent\textbf{(A1) Lipschitzness.}
There exists $L\ge 0$ such that for all $0\le s\le t\le 1$, all conditions $c$, and all $x,y$,
\begin{equation}
\big\| f_\theta(x,t,s,c) - f_\theta(y,t,s,c)\big\|_2 \;\le\; L\,\|x-y\|_2.
\label{eq:app_lipschitz}
\end{equation}

\smallskip
\noindent\textbf{(A2) Local approximation bound for step size $h$.}
For some step size $h>0$, there exist constants $A,B\ge 0$ such that for all $c$, all $x$, and all $t$ satisfying $t-h\ge 0$,
\begin{equation}
\big\| f_\theta(x,t,t-h,c) - f(x,t,t-h,c)\big\|_2 \;\le\; A h + B h^2.
\label{eq:app_local_error}
\end{equation}

\smallskip
\noindent\textbf{(A3) Approximate semigroup consistency.}
There exists $\varepsilon\ge 0$ such that for all $0\le s<\ell<t\le 1$, all $c$, and all $x$,
\begin{equation}
\big\| d_\theta(x;t,\ell,s,c)\big\|_2 \;\le\; \varepsilon.
\label{eq:app_semigroup_eps}
\end{equation}
Then, for any interval $(t,s)$ with macroscopic jump length $\Delta \triangleq t-s$, any integer partition $K\ge 1$, and uniform step size $h=\Delta/K$, the global long-jump error is bounded by:
\begin{align}
\big\| f_\theta(x_t,t,s,c) - f(x_t,t,s,c)\big\|_2
& \;\le\;
(K-1)\varepsilon \;+\; LK\!\left(A\frac{\Delta}{K} + B\frac{\Delta^2}{K^2}\right) \nonumber \\
& \;=\; 
(K-1)\varepsilon \;+\; L\!\left(A\Delta + \frac{B}{K}\Delta^2\right).
\label{eq:app_long_jump_bound}
\end{align}
\end{theorem}

\begin{proof}
We first establish the key one-step decomposition inequality. Fix $s<\ell<t$ and define the exact intermediate state $x_\ell^\star \triangleq f(x_t,t,\ell,c)$. By the true semigroup property \eqref{eq:app_true_semigroup}, we have $f(x_t,t,s,c)=f(x_\ell^\star,\ell,s,c)$. 

By adding and subtracting the intermediate learned terms $f_\theta(x_\ell^\star,\ell,s,c)$ and $f_\theta\!\big(f_\theta(x_t,t,\ell,c),\ell,s,c\big)$, the pointwise approximation error \eqref{eq:app_pointwise_error} cleanly decomposes into three parts:
\begin{align*}
e_\theta(x_t;t,s,c)
&=
\underbrace{f_\theta(x_t,t,s,c) - f_\theta\!\big(f_\theta(x_t,t,\ell,c),\ell,s,c\big)}_{d_\theta(x_t;t,\ell,s,c)}
\\
&\quad +
\underbrace{f_\theta\!\big(f_\theta(x_t,t,\ell,c),\ell,s,c\big) - f_\theta(x_\ell^\star,\ell,s,c)}_{\text{propagation of }e_\theta(x_t;t,\ell,c)}
\\
&\quad +
\underbrace{f_\theta(x_\ell^\star,\ell,s,c) - f(x_\ell^\star,\ell,s,c)}_{e_\theta(x_\ell^\star;\ell,s,c)}.
\end{align*}

\noindent By taking the $L_2$ norm, applying the Lipschitz assumption \eqref{eq:app_lipschitz} to the middle propagation term, and recognizing that $\|f_\theta(x_t,t,\ell,c) - x_\ell^\star\|_2 = \|e_\theta(x_t;t,\ell,c)\|_2$, we obtain:
\begin{equation}
\|e_\theta(x_t;t,s,c)\|_2
\;\le\;
\|d_\theta(x_t;t,\ell,s,c)\|_2 \;+\; L\,\|e_\theta(x_t;t,\ell,c)\|_2 \;+\; \|e_\theta(x_\ell^\star;\ell,s,c)\|_2.
\label{eq:app_recursion_one}
\end{equation}

Now, choose a uniform partition $t=t_K>t_{K-1}>\cdots>t_0=s$ with step size $h=\Delta/K$. Setting the intermediate point $\ell=t_{k-1}$ in \eqref{eq:app_recursion_one} for the interval $(t_k,s)$, and bounding the terms using \eqref{eq:app_semigroup_eps} and \eqref{eq:app_local_error}, gives the following recursion for $k\ge 1$:
$$ \sup_{x,c}\|e_\theta(x;t_k,s,c)\|_2 \;\le\; \varepsilon \;+\; L(Ah + Bh^2) \;+\; \sup_{x,c}\|e_\theta(x;t_{k-1},s,c)\|_2. $$
We adopt the convention that $\varepsilon=0$ when $k=1$, because setting $t_0=s$ causes the semigroup defect to vanish entirely by the boundary condition. Unrolling this recursion over all $K$ steps yields:
\begin{align*}
\sup_{x,c}\|e_\theta(x;t,s,c)\|_2
& \;\le\; 
(K-1)\varepsilon \;+\; KL(Ah+Bh^2) \\
& \;=\;
(K-1)\varepsilon \;+\; L\!\left(A\Delta + \frac{B}{K}\Delta^2\right),
\end{align*}
which verifies the bound in \eqref{eq:app_long_jump_bound}.
\end{proof}

\begin{corollary}[Connection to SoFlow-style bounds and quadratic shrinkage]
\label{cor:quadratic_shrink}
Suppose the local error on step size $h$ satisfies an SoFlow-style bound of the form $Ah+Bh^2$ (as obtained in SoFlow's analysis~\cite{luo2025soflow} via local finite-difference consistency under standard smoothness assumptions). Theorem~\ref{thm:long_jump_error_bound} proves that explicitly enforcing semigroup consistency (e.g., via MVA) reduces the long-jump quadratic error contribution from $O(\Delta^2)$ down to $O(\Delta^2/K)$. This mathematically formalizes why global semigroup constraints are strictly necessary to improve large jump fidelity.
\end{corollary}

\noindent \textbf{Remark.}
Theorem~\ref{thm:long_jump_error_bound} isolates the fundamental mathematical mechanism behind successful few-step generation. It proves that the long-jump error is governed by two distinct forces: 
\begin{align*}
&1) \quad(K\!-\!1)\varepsilon \quad \text{(accumulated semigroup defect across $K\!-\!1$ compositions)}
\\ & 2) \quad  L\!\left(A\Delta + \frac{B}{K}\Delta^2\right) \quad \text{(propagated local approximation error).}
\end{align*}
If the semigroup defect $\varepsilon$ is driven toward zero—which our MVA objective explicitly targets—the dominant long-jump error becomes entirely governed by the local errors. Crucially, Corollary~\ref{cor:quadratic_shrink} shows that the quadratic expansion $\Delta^2$ is \emph{suppressed by $1/K$}. In practice, $K$ acts as a \emph{virtual refinement factor} induced by how densely the intermediate times $\ell$ are sampled during MVA training. This dense intermediate supervision forces large inference jumps to behave like the stabilized composition of many short transitions, even when inference is executed in only a few NFEs.

% \smallskip
% \noindent
% \emph{Role of the assumptions.}
% Assumption~(A1) (Lipschitzness) ensures stability: errors in intermediate states do not get amplified
% arbitrarily when passed through $f_\theta(\cdot,\ell,s,c)$.
% Assumption~(A2) captures the typical local behavior of approximation error for small steps,
% where the $Ah$ term can be viewed as a first-order (bias) component and the $Bh^2$ term reflects
% a higher-order remainder (e.g., from Taylor/local-consistency approximations).
% Assumption~(A3) quantifies global path inconsistency: $\varepsilon$ upper bounds how far the model deviates
% from the exact semigroup identity, i.e., whether a long jump agrees with composing shorter jumps.
\section{Implementation Details}
\label{app:hyper}
Table~\ref{tab:hyperparameter} summarizes the model hyperparameters used across all experiments.

\subsection{Model}
Our base architecture follows Wan~\cite{wan2025wan} and incorporates the sparse--dense residual design of
SPRINT~\cite{park2025sprint}. The network is structured into an \emph{encoder}, a compute-heavy \emph{middle stack}, and a
\emph{decoder}, comprising $N_e$, $N_m$, and $N_d$ transformer blocks, respectively. This SPRINT-style encoder--decoder
split is connected by a long residual path that bridges early dense features directly to the decoder
(\cref{fig:arch}).
Following latent diffusion frameworks~\cite{wan2025wan}, we operate entirely in a compressed latent space to reduce
compute. The 3D-VAE encodes videos with a $4\times 8\times 8$ (temporal $\times$ spatial) compression factor. To maximize
training throughput, we use the H3AE~\cite{wu2025h3ae} autoencoder during training; it shares the same latent space as the
default Wan-2.1 VAE~\cite{wan2025wan} but provides substantially faster encoding.
Given latent tensors of shape $T'\times H'\times W'$, we configure the local branch in GLGA with a sliding-window size
$|\mathcal{W}| = m \times H' \times W'$, i.e., each window spans approximately $m$ consecutive latent frames. Increasing $m$ enlarges
the local receptive field and improves expressivity, but incurs higher computation. Empirically, we find $m\in\{4,6\}$
offers a good trade-off between preserving fine-grained spatiotemporal detail and maintaining high training throughput.

% Preamble:
% \usepackage{booktabs}
% \usepackage{makecell}
% \usepackage{xcolor,colortbl}

\begin{table}[t]
\centering
\caption{\textbf{Hyperparameters used in our experiments.} Kinetics-700 models (B/XL/XXL) and Text-to-Video model (1.5B).}
\label{tab:hyperparameter}
\setlength{\tabcolsep}{6pt}
\renewcommand{\arraystretch}{1.12}
\resizebox{\linewidth}{!}{%
\begin{tabular}{lcccc}
\toprule
& \makecell{\modelname-B \\ (Fig.~\ref{fig:performance_comparison}, Tab.~\ref{tab:attn_gating_ablation}, \ref{tab:training_objective})}  
& \makecell{\modelname-XL \\ (Fig.~\ref{fig:performance_comparison}--\ref{fig:mva}, Tab.~\ref{tab:fvd_comparison})} 
&  \makecell{\modelname-XXL \\ (Fig.~\ref{fig:performance_comparison})} 
&  \makecell{\modelname-1.5B \\ (Fig.~\ref{fig:efficiency}, Tab.~\ref{tab:vbench})} \\
\midrule
Params.                 & 128M & 529M & 3B & 1.5B \\
Patch size                 & $2\times2\times1$ & $2\times2\times1$ & $2\times2\times1$ & $2\times2\times1$ \\
Total Num. blocks   & 14                & 28                & 36                & 28 \\
$N_e$                      & 2                 & 3                 & 4                 & 3 \\
$N_m$                      & 10                & 22                & 28                & 22 \\
$N_d$                      & 2                 & 3                 & 4                 & 3 \\
Hidden dims                & 768               & 1152              & 2304              & 1536 \\
FFN dims                   & 3072              & 4608              & 8064              & 7680 \\
Num.\ heads                & 12                & 16                & 24                & 12 \\
Softmax-attn ratio         & 6:1               & 8:1               & 8:1               & 8:1 \\
$w_{\text{min}}$           & 2.5               & 2.5               & 2.5               & 2.5 \\
$w_{\text{max}}$           & 4.5               & 4.5               & 4.5               & 4.5 \\
$w_{\text{inf}}$           & 4.0               & 4.0               & 4.0               & 4.0 \\
$m$           & 4               & 4               & 4               & 6 \\
$(u_t, \sigma_t, \text{shift}_t)$           & (0, 1, 3)               & (0, 1, 3)               & (0, 1, 3)               & (0, 1, 3) \\
$(u_s, \sigma_s, \text{shift}_s)$           & (-1, 0.8, 1)               & (-1, 0.8, 1)               & (-1, 0.8, 1)               & (-1, 0.8, 1) \\
\rowcolor{gray!15}\multicolumn{5}{l}{\textit{SoFlow hyperparameters}} \\
Velocity mix ratio         & 0.25              & 0.25              & 0.25              & 0.25 \\
Loss weight coefficient    & 1.0               & 1.0               & 1.0               & 1.0 \\
$l\rightarrow t$ schedule & Exponential               & Exponential               & Exponential               & Exponential \\
\bottomrule
\end{tabular}%
}
\end{table}

\smallskip
\noindent \textbf{Text-to-Video (T2V).}
In the T2V setting, each transformer block applies self-attention, cross-attention, and a feed-forward network (FFN) in
sequence. To reduce compute while retaining expressivity, we replace the quadratic self-attention layers with our
\emph{Gated Local--Global Attention} (GLGA). Unless otherwise stated, we use sparse softmax interleaving, inserting one
standard full softmax-attention layer every 8 GLGA blocks. For text conditioning, we use a pre-trained T5 encoder~\cite{raffel2020exploring}
to obtain prompt embeddings. We additionally introduce a guidance embedding to condition the network explicitly on the
guidance scale $w$; this embedding is summed with the timestep embedding and injected into all blocks via Adaptive Layer
Normalization (AdaLN)~\cite{peebles2023scalable}.

\smallskip
\noindent \textbf{Class-conditional Kinetics.}
For class-conditional generation on Kinetics, we adapt the T2V backbone by removing cross-attention and conditioning on
the class label instead. Specifically, we map the discrete label to a learned class embedding, sum it with the timestep
embedding, and inject the resulting conditioning into the transformer blocks through AdaLN.

\subsection{Loss Objective}
We largely follow the default training setup of SoFlow~\cite{luo2025soflow}, including its adaptive loss weighting,
timestep sampling distribution (\eg, log-normal distribution), and velocity-mix ratio. Following SPRINT~\cite{park2025sprint},
we additionally use structured group-wise token subsampling to ensure local coverage while maintaining global randomness.

\smallskip
\noindent \textbf{Timestep sampling.}
Following SoFlow, we sample timesteps $t$ (and $s$) by drawing a log-normal latent variable and applying a sigmoid
transform, i.e., $t=\sigma(u)$ with $u\sim\mathcal{N}(\mu_t,\sigma_t^2)$, and then apply a shift to bias samples toward
earlier times:
\begin{equation}
t'=\frac{\texttt{shift}\cdot t}{1+(\texttt{shift}-1)t}.
\end{equation}
We use this distribution throughout training. For solution-consistency sampling, we enforce the temporal ordering
$t>s>l$. For the discrete timesteps used by MVA, we instead sample 64 points ($K=64$) uniformly from $[0,1]$ and apply the same shift with $\texttt{shift}=3$.

\smallskip
\noindent \textbf{Path-drop guided training.}
Our PDG regime requires the model to also predict the unconditional velocity field through an inexpensive \emph{weak
path}. During training, with probability $p_{\text{drop}}=0.2$, we replace the conditioning $c$ with the null label
$\emptyset$ and bypass the compute-heavy middle stack. When path-drop is activated, we update the weak-path parameters
using the unconditional flow-matching objective (Eq.~\ref{eq:soflow_fm}) by setting the guidance scale to $w=0$ (Eq.~\ref{eq:pdg_cfg}).

\smallskip
\noindent \textbf{Guidance scale.}
Unlike standard SoFlow, which uses a fixed classifier-free guidance scale during training, we sample the guidance scale
per batch as $w \sim \mathcal{U}(w_{\min}, w_{\max})$. This continuous sampling discourages overfitting to a single
magnitude and enables flexible control over guidance strength at inference time.

\subsection{Inference}
Since $f_\theta(\rvx_t,t,s,c)$ parameterizes the ODE solution map from time $t$ to $s$, we sample
$\rvx_1 \sim \mathcal{N}(0,I)$ and obtain a one-step sample by evaluating the solution function at $(t,s)=(1,0)$:
$\hat{\rvx}_0 = f_\theta(\rvx_1,1,0,c)$.
Following the multi-step procedure used in consistency models~\cite{song2023consistency}, we also support multi-step $N$
inference by composing $f_\theta$ over intermediate time points: starting from $\rvx_1$, we repeatedly re-noise the
current prediction to the next time level and apply the corresponding solution map. For inference timesteps, we sample
$N$ points uniformly from $[0,1]$ and apply time shifting with $\texttt{shift}=3$. Unless otherwise stated, we use a fixed
guidance scale $w_{\text{inf}}$ during inference. Refer to Algorithm~\ref{alg:inference} for full pseudo-code.

\begin{algorithm}[t]
\caption{Training \modelname}
\label{alg:training}
\begin{algorithmic}[1]
\Require PDG drop probability $p_{\text{drop}}$, batch partition ratios $a, b$.
\Require Model $f_\theta$, guidance scale $w$.
\While{not converged}
    \State Sample clean videos $x_0$ and conditions $c$
    \State Split the data batch for losses by ratios $1-a-b,\; a,\; b$:
    \Statex \hspace{1.5em} $x_0= (x_0^{\text{FM}}, x_0^{\text{SCM}}, x_0^{\text{MVA}}),$ \quad $c= (c^{\text{FM}}, c^{\text{SCM}}, c^{\text{MVA}})$
    \State Sample times $t^{\text{FM}}$, $(t^{\text{SCM}}, s^{\text{SCM}}, \ell^{\text{SCM}})$, and $(t^{\text{MVA}}, s^{\text{MVA}}, \ell^{\text{MVA}})$
    \State Compute $x_t^{\text{FM}}, v_t^{\text{FM}}, x_t^{\text{SCM}}, v_t^{\text{SCM}}, x_t^{\text{MVA}}$ using the Flow Matching framework
    \State Compute $v_{\text{PDG}}^{\text{FM}}$ and $v_{\text{PDG}}^{\text{SCM}}$ using Eq.~\ref{eq:pdg_cfg} and $w$
    \State Randomly replace $v_{\text{PDG}}^{\text{FM}}$ with $v_t^{\text{FM}}$ with a probability of $p_{\text{drop}}$
    \State Replace $c^{\text{FM}}$ with the empty label $\emptyset$ correspondingly
    \State Compute $\mathcal{L}_{\text{FM}}$ via Eq.~\ref{eq:soflow_fm} by replacing $v$ with $v_{\text{PDG}}^{\text{FM}}$
    \State Compute $\mathcal{L}_{\text{SCM}}$ via Eq.~\ref{eq:soflow_scm} by replacing $v$ with $v_{\text{PDG}}^{\text{SCM}}$
    \State Compute $\mathcal{L}_{\text{MVA}}$ via Eq.~\ref{eq:mva}
    \State Aggregate loss: $\mathcal{L} \leftarrow a\mathcal{L}_{\text{FM}} + b\mathcal{L}_{\text{SCM}} + (1-a-b)\mathcal{L}_{\text{MVA}}$
    \State Update $f_\theta$ via gradient descent on $\mathcal{L}$
\EndWhile
\end{algorithmic}
\end{algorithm}
\begin{algorithm}[t]
\caption{Inference of \modelname}
\label{alg:inference}
\begin{algorithmic}[1]
\Require Trained solution map $f_\theta$, condition $c$, guidance scale $w$.
\Require Number of inference steps $N$ (e.g., $N=4$).
\State Sample initial noise $x \sim \mathcal{N}(0, \mathbf{I})$
\State Define a time schedule $1 = t_N > t_{N-1} > \dots > t_0 = 0$
\For{$n = N$ \textbf{down to} $1$}
    \State $t \leftarrow t_n$
    \State $t_{\text{next}} \leftarrow t_{n-1}$
    \State Predict clean data $\hat{x}_0 \leftarrow f_\theta(x, t, 0, c, w)$ \Comment{Direct jump to $t=0$}
    \If{$t_{\text{next}} > 0$}
        \State Sample fresh noise $z \sim \mathcal{N}(0, \mathbf{I})$
        \State $x \leftarrow t_{\text{next}} z + (1 - t_{\text{next}}) \hat{x}_0$ \Comment{Renoise for next jump}
    \Else
        \State $x \leftarrow \hat{x}_0$ \Comment{Final clean output}
    \EndIf
\EndFor
\State \textbf{return} Generated clean video $x$
\end{algorithmic}
\end{algorithm}

\section{Training Procedure}

\subsection{Kinetics}
For class-conditional generation on Kinetics, we train \modelname at $256\times256\times31$ using 8 NVIDIA A100 GPUs.
For the first 80\% of training, we apply an aggressive 75\% token-drop ratio to maximize throughput. During this phase,
we disable the Mean-Velocity Additivity (MVA) regularizer to allow the base solution-flow trajectories to stabilize. We
use a batch-split ratio of $[0.75,\,0.25]$ for the flow-matching and solution-consistency losses. For the remaining 20\%
of training, we disable token dropping (0\% drop) to close the train--inference mismatch and enable MVA to enforce global
long-jump consistency on the stabilized trajectories. We use a constant learning rate of $1\times10^{-4}$ with a global
batch size of 256.

\subsection{Text-to-Video}
To train \modelname for text-to-video generation entirely from scratch, we use a three-stage progressive training
schedule on 128 NVIDIA A100 GPUs. We initialize all weights randomly and do not rely on pre-trained image or video
foundation models. Following prior work~\cite{wan2025wan,chen2025sana,jin2024pyramidal}, we adopt a low-to-high
resolution curriculum to accelerate early convergence and reduce compute. We also perform joint image--video training
throughout, maintaining a 10\% image ratio per batch to strengthen spatial modeling.

\smallskip
\noindent{\textbf{1) Low-resolution training with token dropping.}}
We first train at 256p with 81 frames to learn basic spatiotemporal dynamics and text alignment. We apply 75\% token
dropping for efficiency, and interleave dense full-token updates at a 9:1 iteration ratio (one dense update every nine
sparse updates) to mitigate overfitting to sparse representations. We train for 150K iterations with learning rate
$1\times10^{-4}$ and global batch size 2048. MVA is disabled in this stage.

\smallskip
\noindent{\textbf{2) High-resolution training with token dropping.}}
We then increase the spatial resolution to 480p while retaining 75\% token dropping. We train for an additional 50K
iterations with learning rate $1\times10^{-4}$ and global batch size 1024. We keep MVA disabled while the model adapts to
the longer sequences and higher-frequency details.

\smallskip
\noindent{\textbf{3) High-resolution full-token annealing.}}
Finally, we disable token dropping (0\% drop) and fine-tune on dense 480p tokens to eliminate residual train--inference
mismatch. We train for 50K iterations with a reduced learning rate of $2\times10^{-5}$ and global batch size 1024.
Crucially, we enable MVA in this stage to enforce global long-jump consistency on the stabilized high-resolution
trajectories, improving fidelity for few-step inference.

\begin{figure}[t]
    \centering
    \includegraphics[width=0.8\linewidth]{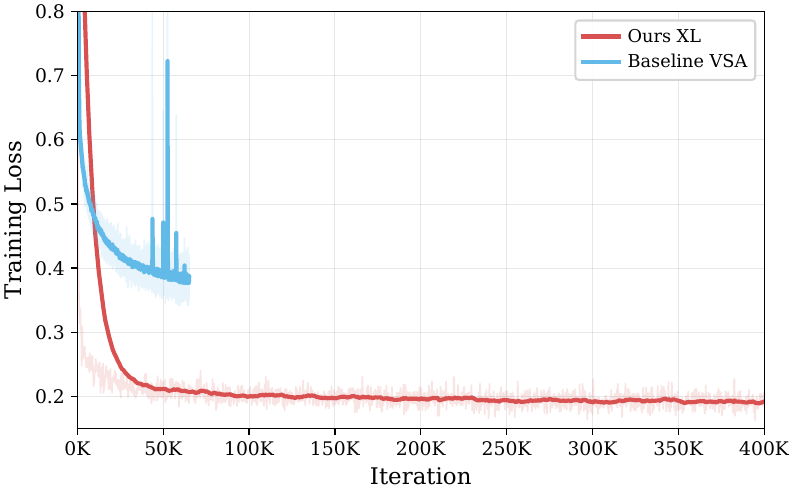}
    \caption{\textbf{Training loss over 400K iterations.} Our \modelname (red) converges rapidly and remains stable throughout training. In contrast, the Video Sparse Attention (VSA) baseline (blue) exhibits pronounced optimization instability, with frequent loss spikes early in training and reduced convergence reliability under the same setup. These results suggest that VSA is more sensitive to training configurations, whereas \modelname trains robustly without specialized stabilization or hyperparameter tuning.}
    \label{fig:vsa}
\end{figure}

\begin{figure}[t]
    \centering
    \includegraphics[width=\linewidth]{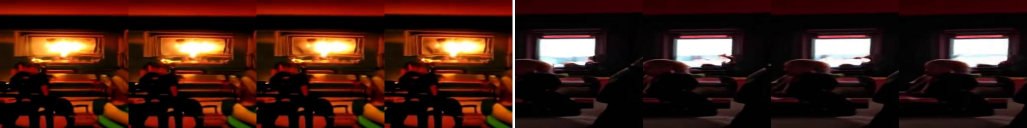} \smallskip
    \includegraphics[width=\linewidth]{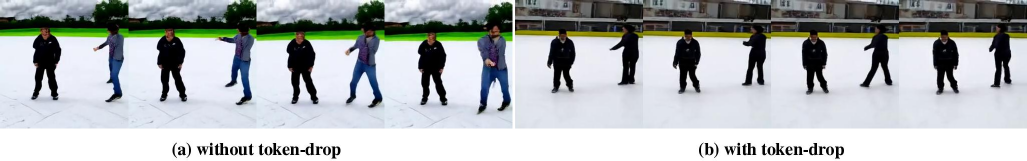}
    \caption{\textbf{Visual ablation of token-drop training.} We compare models trained (a) entirely on dense grids without token dropping, and (b) utilizing our aggressive 75\% token-dropping framework. Despite the massive reduction in training compute, the model trained with token dropping rigorously preserves spatial fidelity, fine-grained details, and overall temporal consistency, demonstrating that our sparse training recipe does not degrade generative quality.}
    \label{fig:token-drop}
\end{figure}

\begin{figure}[t]
    \centering
    \includegraphics[width=\linewidth]{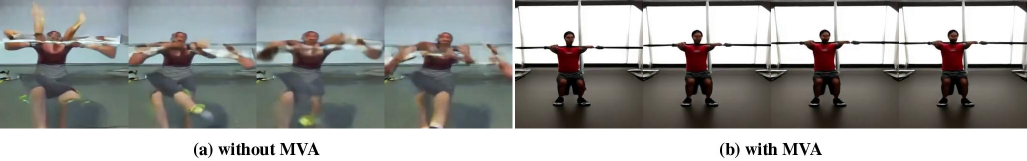}
    \caption{\textbf{Visual ablation of Mean-Velocity Additivity (MVA).} Both examples use 4-step sampling. (a) Without MVA, integration errors accumulate across steps, leading to distortion, motion blur, and structural collapse. (b) With MVA, enforcing global semigroup consistency improves sharp structure, consistent subject identity and motion.}
    \label{fig:mva}
\end{figure}

\begin{table}[t]
\centering
\caption{\textbf{VBench results for 480p text-to-video generation.} All baselines are re-evaluated using the Diffusers library with our augmented prompts for consistent comparison. $^\star$ denotes models specifically fine-tuned on Wan-14B-generated images. $^\dagger$ denotes unofficial checkpoints due to restricted access to the official
open-source release.}
\setlength{\tabcolsep}{8pt}
\renewcommand{\arraystretch}{1.15}

\resizebox{\textwidth}{!}{%
\begin{tabular}{l c c c c c c c}
\toprule
\multirow{2}{*}{\textbf{Methods}} &
\multirow{2}{*}{\textbf{NFE}} &
\multirow{2}{*}{\textbf{Latency} (s)} &
\multirow{2}{*}{\textbf{Speedup}} &
\multirow{2}{*}{\textbf{\#Params} (B)} &
\multicolumn{3}{c}{\textbf{Evaluation scores} $\uparrow$} \\
\cmidrule(lr){6-8}
& & & & & \textbf{Total} & \textbf{Quality} & \textbf{Semantic} \\
\midrule

\rowcolor{gray!15}
\multicolumn{8}{l}{\textit{Multi-step DiT}} \\
% Wan2.2    & 40$\times$2 & -- & -- & 14B &   &  &  \\
Wan2.1    & 40$\times$2 & 732 & 1.0$\times$ & 14B &  \textbf{82.34} &  \textbf{83.77} &  \textbf{76.65} \\
Wan2.1    & 40$\times$2 & 147 & 4.97$\times$ & 1.3B & 82.11 & 83.77 & 75.46 \\

% \rowcolor{gray!15} \multicolumn{8}{l}{\textit{Multi-step Efficient DiT}} \\
SANA-Video & 40$\times$2 & {79.03} & 9.26$\times$ & {2B} &
{81.93} & {84.11} & {73.23} \\
%Adaptor   & -- & \textbf{60} & \textbf{8.0$\times$} & \textbf{2B} & \underline{83.71} & \underline{84.35} & \textbf{81.35} \\
% STA       & 50$\times$2 & {--} & {--} & {1.3B} &
% {--} & {--} & {--} \\

\rowcolor{gray!15}
\multicolumn{8}{l}{\textit{Few-step DiT with teacher-distillation}} \\
Wan2.1 + VSA$^\star$       & 4 & {11.28} & 64.9$\times$ & {1.3B} &
{81.28} & {83.36} & {72.97} \\
Wan2.1 + rCM$^\dagger$  & 4 & {12.61} & 58.0$\times$ & {1.3B} &
 79.36 & 83.75 & 61.77 \\

\midrule
\rowcolor{gray!15}
\multicolumn{8}{l}{\textit{Few-step Efficient DiT from-scratch}} \\
\textbf{Ours} & 4 & \textbf{10.09} & \textbf{72.5$\times$} & {1.5B} &
{79.69} & {83.53} & {64.34} \\
\bottomrule
\end{tabular}%
} % end resizebox
\label{tab:vbench}
\end{table}

\section{Additional Results}

\subsection{Comparison to Video Sparse Attention}
Video Sparse Attention (VSA)~\cite{zhang2025vsa} proposes an adaptive-sparse attention mechanism designed to be trainable from scratch. To rigorously compare our GLGA backbone against VSA in a controlled setting, we train a VSA-equipped model on the Kinetics-700 dataset using the exact same hyperparameters and compute budget as our \modelname-XL in Tab.~\ref{tab:hyperparameter}. Following the VSA paper, we use 64 blocks ($B{=}64$) and select 32 tiles for sparse attention.
Fig.~\ref{fig:vsa} shows that VSA is highly unstable under this standard setup: the VSA baseline (blue) exhibits large loss spikes early in training and fails to converge reliably. We attribute this behavior to the brittleness of \emph{adaptive} sparsity during early optimization. Because VSA learns data-dependent routing to determine its sparsity patterns, the routing decisions are poorly calibrated when representations are still immature, which can lead to unstable training. In contrast, our \modelname (red) converges quickly and remains stable without special schedules or optimized hyperparameters.

\subsection{Token-Drop Training}
In Fig.~\ref{fig:token-drop}, we visually ablate token-drop training by comparing (a) a model trained entirely on dense
grids (no token dropping) and (b) a model trained with our aggressive 75\% token-dropping framework for the first 80\% of
training, followed by 20\% full-token training. Despite the substantial reduction in training compute, the token-dropped
model preserves spatial fidelity, fine-grained detail, and overall temporal consistency. The resulting samples remain
crisp and structurally coherent, with stable motion over time, indicating that our sparse training recipe does not
compromise generative quality relative to dense training.

\subsection{Mean-Velocity Additivity}
Fig.~\ref{fig:mva} provides a visual ablation of Mean-Velocity Additivity (MVA) regularization for few-step
generation. We compare videos generated in 4 steps using models trained with and without MVA. Without MVA, the
large inference step size $\Delta$ amplifies the accumulation of local integration errors (the semigroup defect; see
Corollary~\ref{cor:quadratic_shrink}), yielding substantial trajectory drift. Qualitatively, this appears as structural
collapse, anatomical distortions, and pronounced motion blur. By explicitly enforcing global path consistency during training, MVA suppresses these large $O(\Delta^2)$ integration errors. As a result, the MVA-regularized model preserves sharper structure, more coherent subject identity, and stable motion kinematics under 4-step sampling.

\subsection{Text-to-Video VBench Evaluation}
\label{sec:vbench_results}

Tab.~\ref{tab:vbench} presents the VBench evaluation results for 480p text-to-video generation. Our 1.5B model achieves a total score of 79.69, operating at a state-of-the-art latency of just 10.09 seconds (a 72.5$\times$ speedup over the Wan2.1 14B baseline). While our absolute semantic scores of 64.34 are lower than the leading multi-step baselines (Wan 2.1 and Sana-Video), it is critical to contextualize these metrics within the different training paradigms. 

\smallskip
\noindent \textbf{Discrepancies in initialization and training compute.}
The baselines in Tab.~\ref{tab:vbench} benefit from large computational advantages and strong pre-trained priors.
Distillation-based methods (Wan2.1+VSA and Wan2.1+rCM) do not learn the generative dynamics from scratch; instead, they
inherit the capabilities of a fully trained foundation teacher. Likewise, Sana-Video leverages heavy text-to-image
pretraining before video fine-tuning. In contrast, our method targets the \emph{few-step efficient DiT-from-scratch}
regime: we initialize the network randomly (without pre-trained image or video weights) and train with a fraction of the compute used by industrial-scale baselines like Wan 2.1.

\smallskip
\noindent \textbf{Algorithmic efficacy in controlled settings.}
To isolate and prove the fundamental algorithmic superiority of our approach, we emphasize our controlled comparisons on Kinetics-700 (Fig.~\ref{fig:qual_figure} and Tab.~\ref{tab:fvd_comparison}). When initialization and compute
are matched, our PDG-SFM framework with GLGA consistently outperforms standard baselines, demonstrating improved few-step generation fidelity without relying on large-scale pretraining or teacher distillation.                                                                                                                                                                           
                                                                                                                                                                                                                                                                                                                                                                                                                                                                                                           
\clearpage
\section{Qualitative results}
In Figs.~\ref{fig:qual_figure1}, \ref{fig:qual_figure2}, and \ref{fig:qual_figure3}, we present generation results of our \modelname trained on Kinetics.
In Figs.~\ref{fig:qual_figure4}, \ref{fig:qual_figure5}, and \ref{fig:qual_figure6}, we present text-to-video generation results.

\begin{figure}[h]
    \centering
    \includegraphics[width=\linewidth]{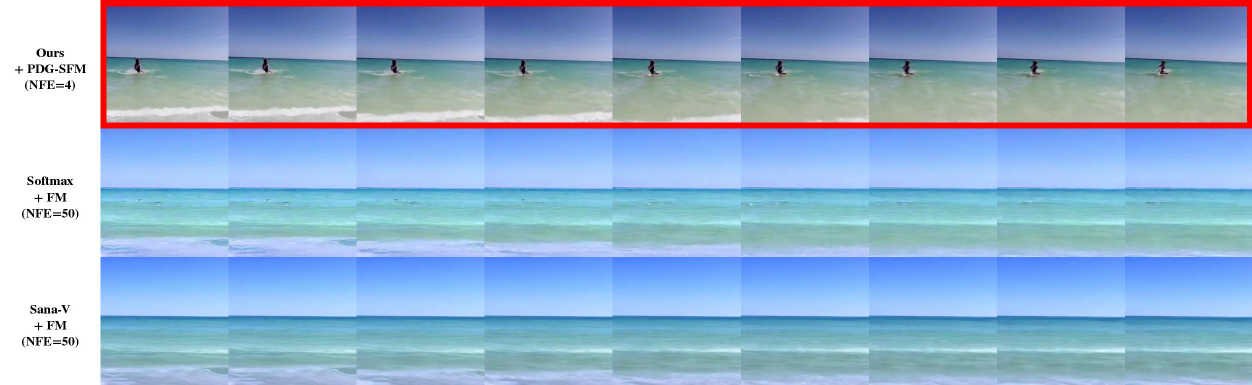} \par\medskip
    \includegraphics[width=\linewidth]{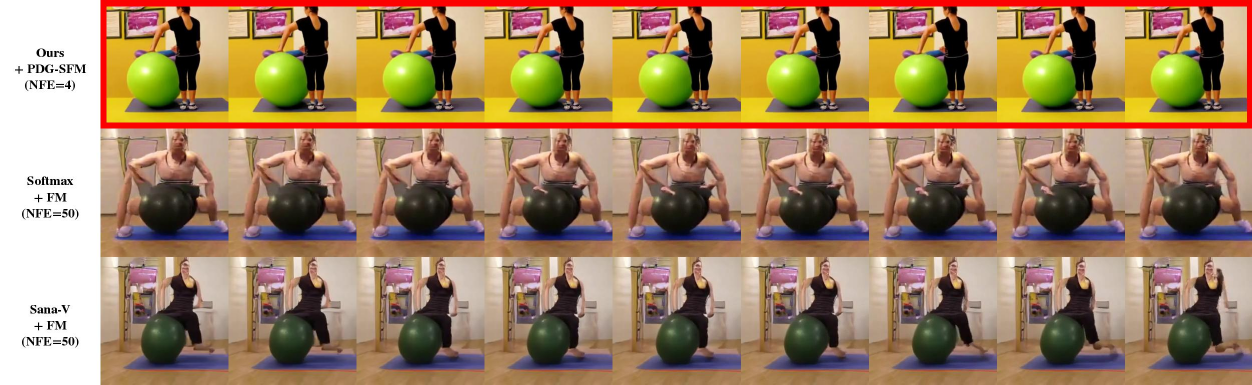} \par\medskip
    \includegraphics[width=\linewidth]{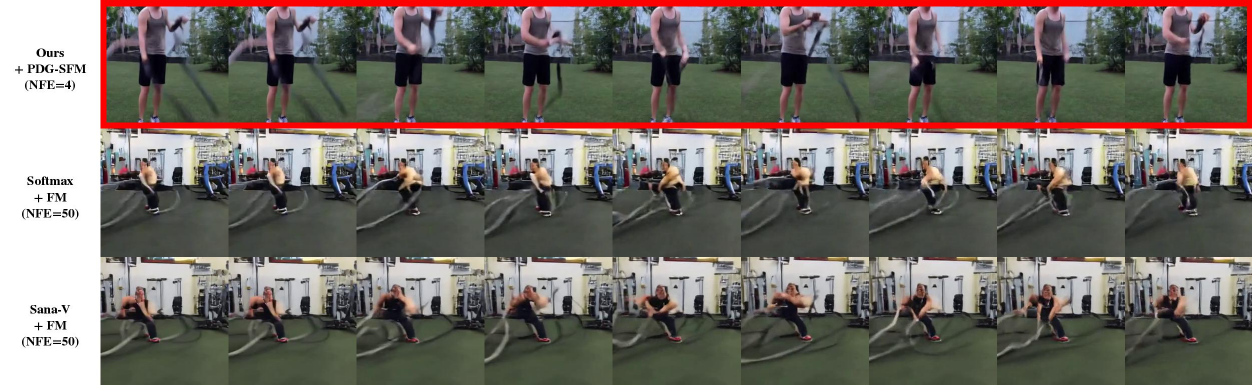}
    \caption{\textbf{Qualitative comparison on Kinetics-700} on 400K training iterations.}
    \label{fig:qual_figure1}
\end{figure}

\begin{figure}[h]
    \centering
    \includegraphics[width=\linewidth]{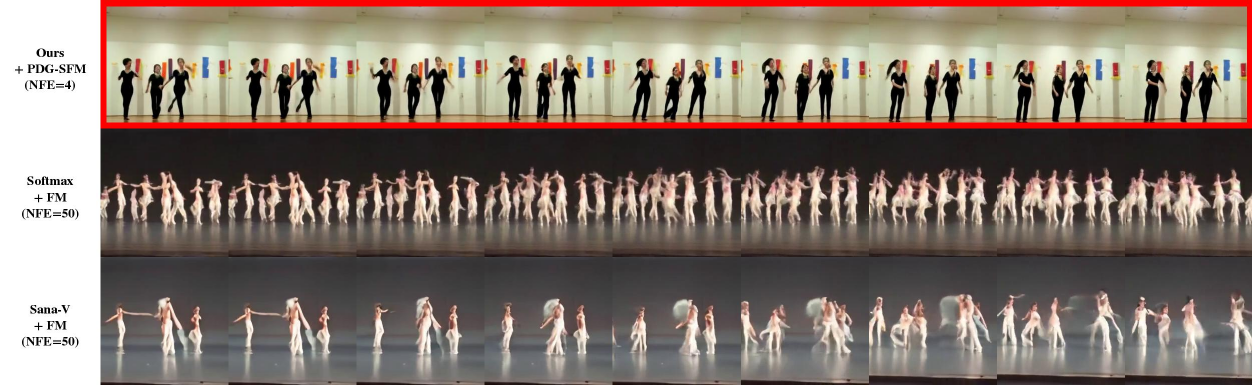} \par\medskip
    \includegraphics[width=\linewidth]{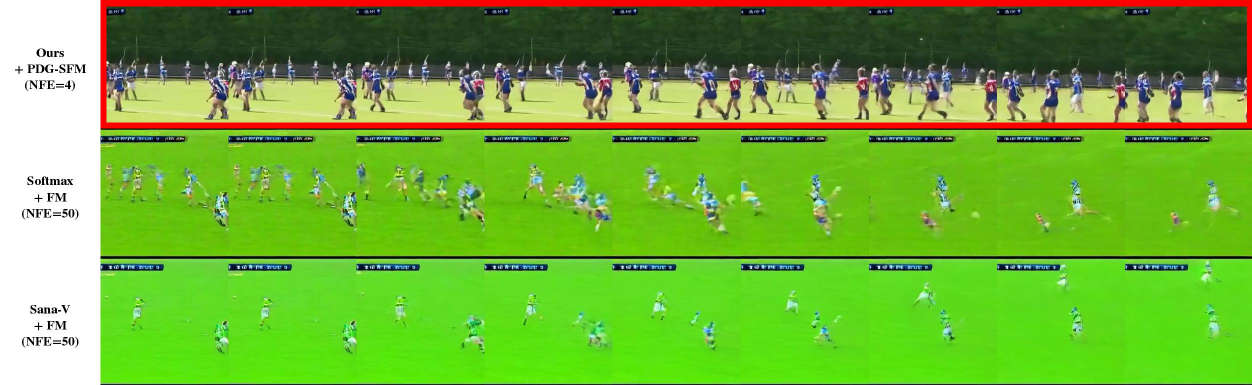} \par\medskip
    \includegraphics[width=\linewidth]{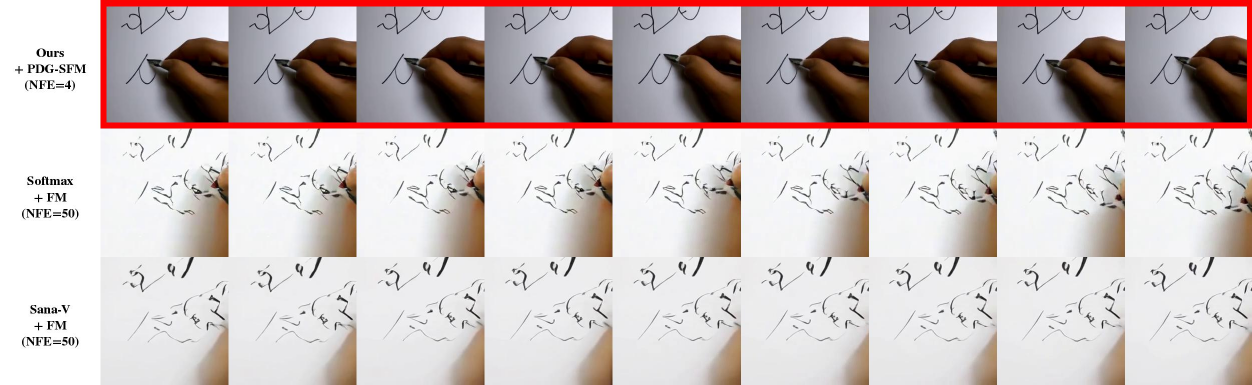}
    \caption{\textbf{Qualitative comparison on Kinetics-700} on 400K training iterations.}
    \label{fig:qual_figure2}
\end{figure}

\begin{figure}[h]
    \centering
    \includegraphics[width=\linewidth]{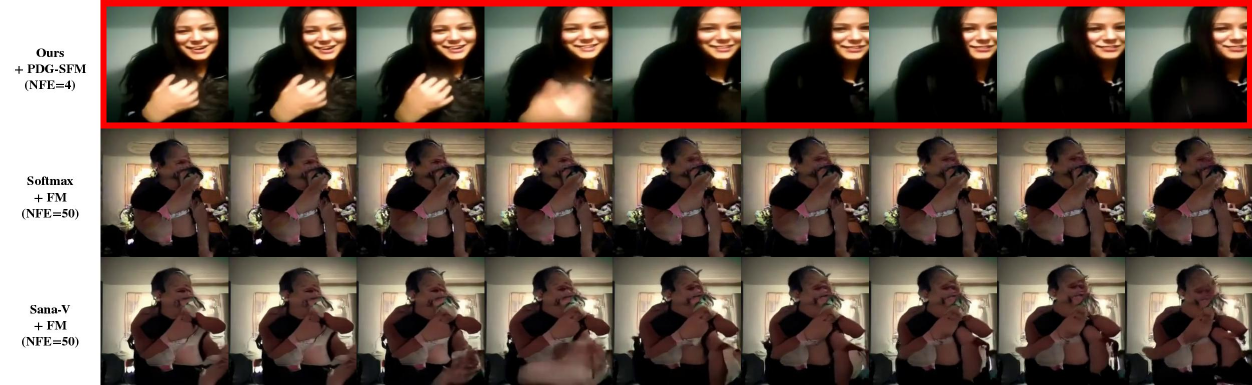} \par\medskip
    \includegraphics[width=\linewidth]{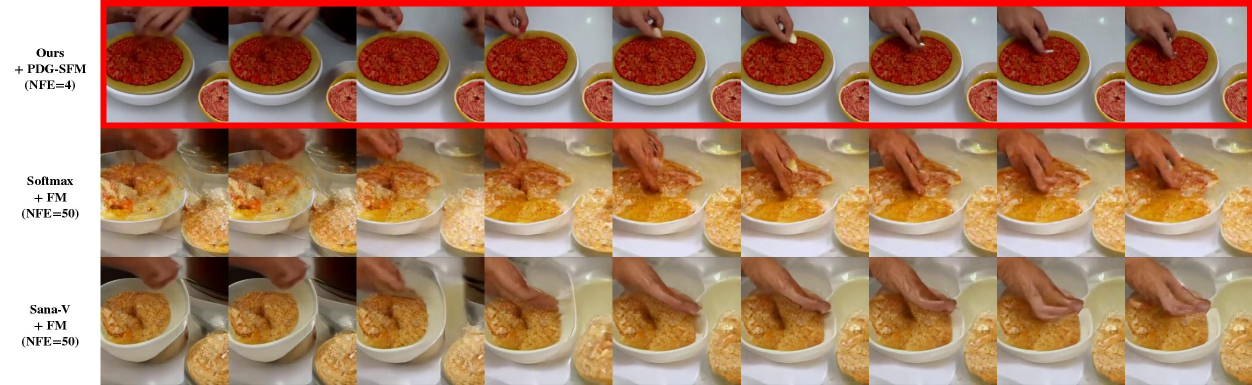} \par\medskip
    \includegraphics[width=\linewidth]{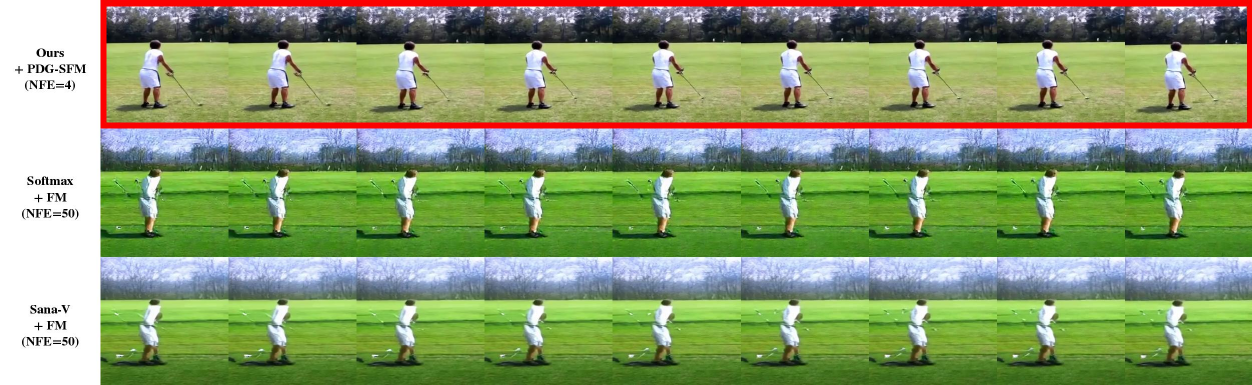}
    \caption{\textbf{Qualitative comparison on Kinetics-700} on 400K training iterations.}
    \label{fig:qual_figure3}
\end{figure}

\begin{figure}[h]
    \centering
    \includegraphics[width=\linewidth]{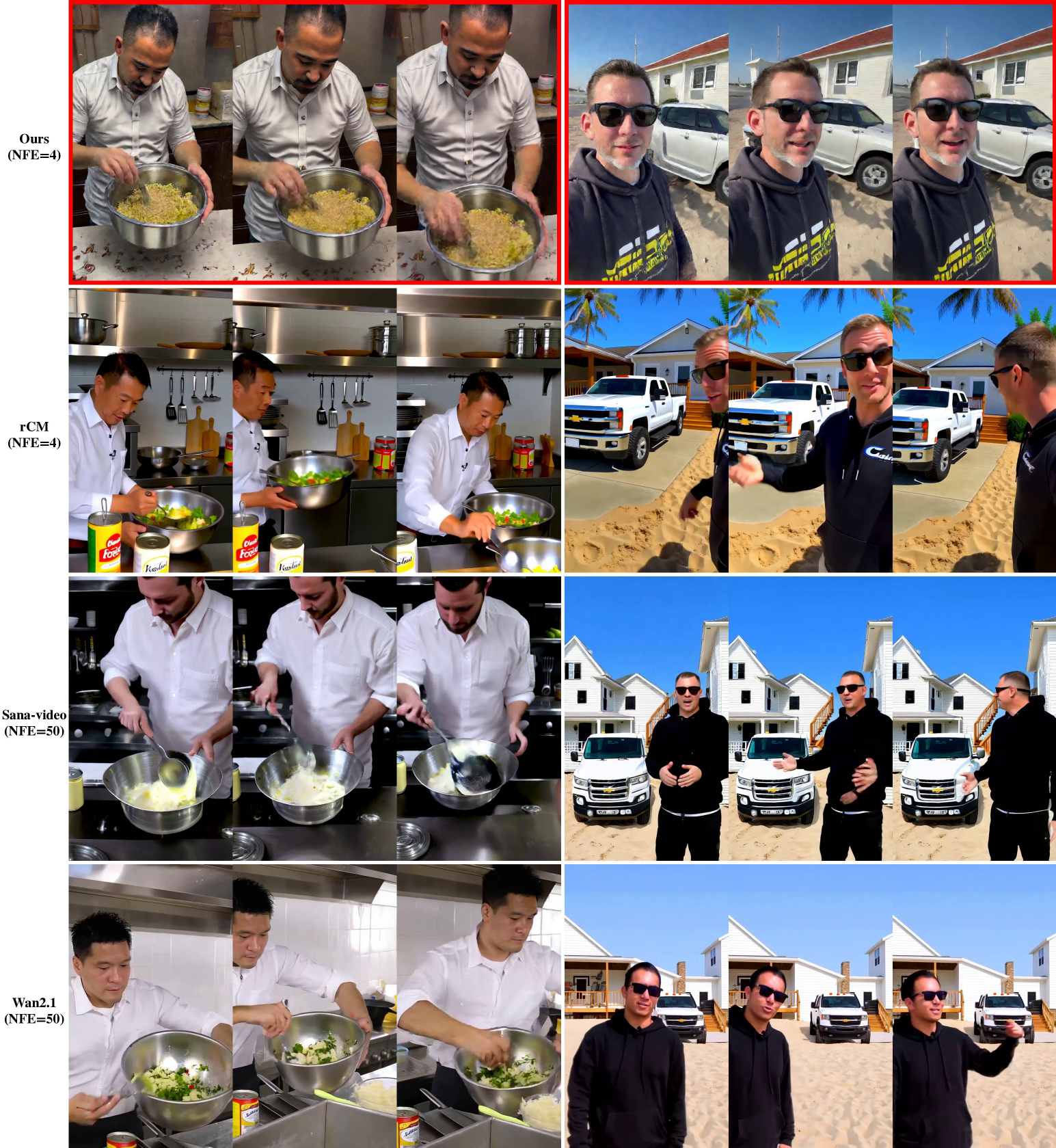}
    \caption{\textbf{Qualitative comparison on 480p text-to-video}. Left column prompt: \emph{In the video, a man is seen in a kitchen, preparing food. He is wearing a white shirt and is focused on his task. The man is holding a large metal bowl filled with what appears to be a mixture of ingredients, possibly a salad or a mixture for a dish. He is using a spoon to mix the contents of the bowl, ensuring that everything is well combined.} Right column prompt: \emph{In the video, a man is standing on a sandy beach, wearing a black hoodie and sunglasses. He is speaking to the camera, gesturing towards a white Chevrolet truck parked in front of a house. The truck is parked on a concrete driveway, and the house has a white exterior with a wooden deck and stairs leading up to the front door. The sky is clear and blue, indicating a sunny day. The man appears to be enjoying his time at the beach, and the truck is parked in a way that suggests it is ready for a road trip.}}
    \label{fig:qual_figure4}
\end{figure}

\begin{figure}[h]
    \centering
    \includegraphics[width=\linewidth]{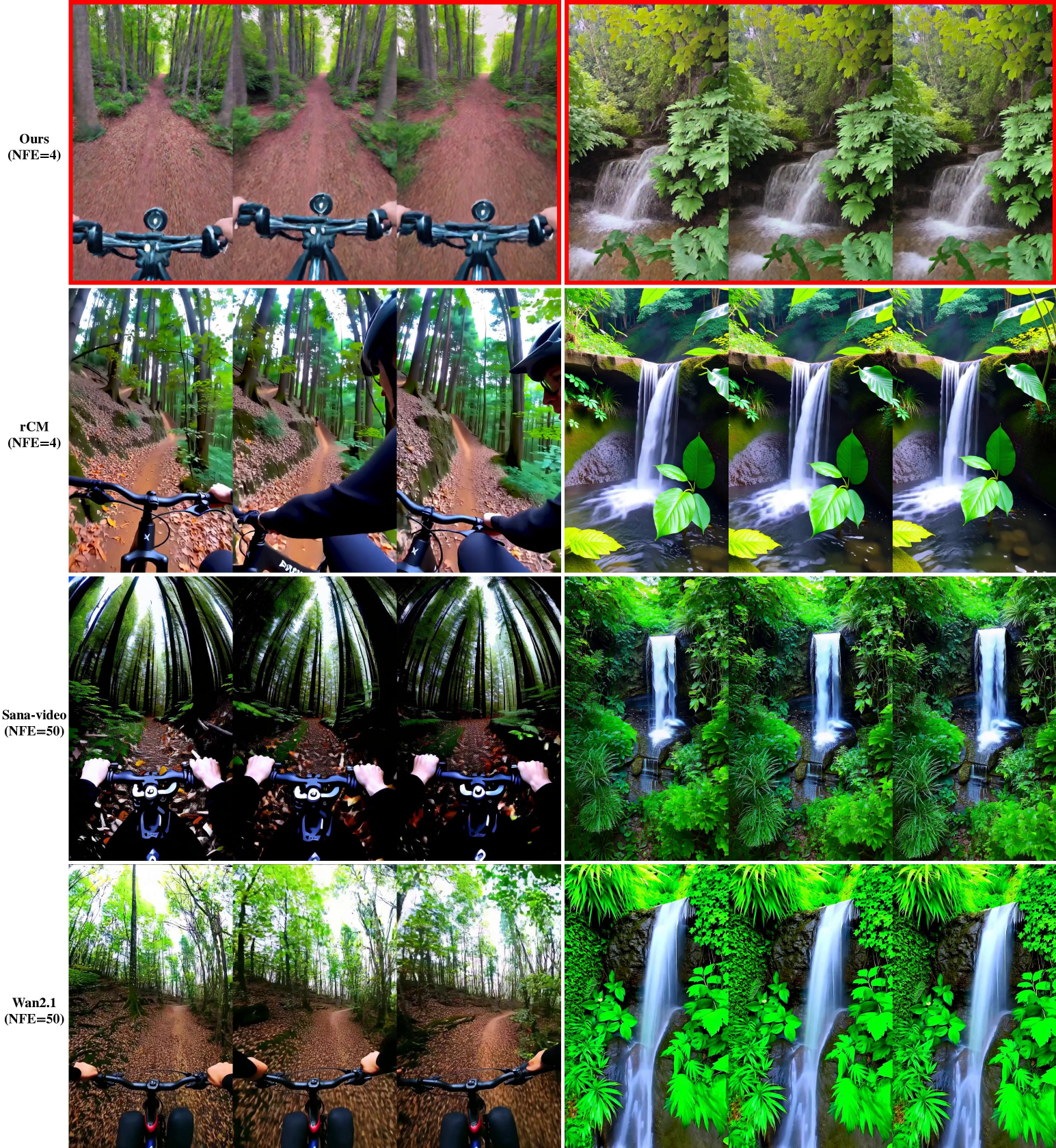}
    \caption{\textbf{Qualitative comparison on 480p text-to-video}. Left column prompt: \emph{"The video captures a thrilling mountain biking experience in a dense forest. The perspective is from the rider's point of view, providing an immersive experience of the journey. The rider is seen maneuvering through a narrow, winding trail covered with fallen leaves and twigs, indicating a fall season setting. The trail is surrounded by tall trees with green foliage, creating a canopy overhead. The rider's hands are firmly gripping the handlebars, guiding the bike through the uneven terrain."} Right column prompt: \emph{"The video captures a serene and picturesque scene of a small waterfall nestled within a lush, green forest. The waterfall cascades down a rocky ledge, surrounded by a variety of vibrant green plants and foliage. The water flows smoothly, creating a gentle, soothing sound that adds to the tranquil atmosphere. The surrounding vegetation includes large, broad leaves and smaller, delicate plants, creating a dense and verdant environment."}}
    \label{fig:qual_figure5}
\end{figure}

\begin{figure}[h]
    \centering
    \includegraphics[width=\linewidth]{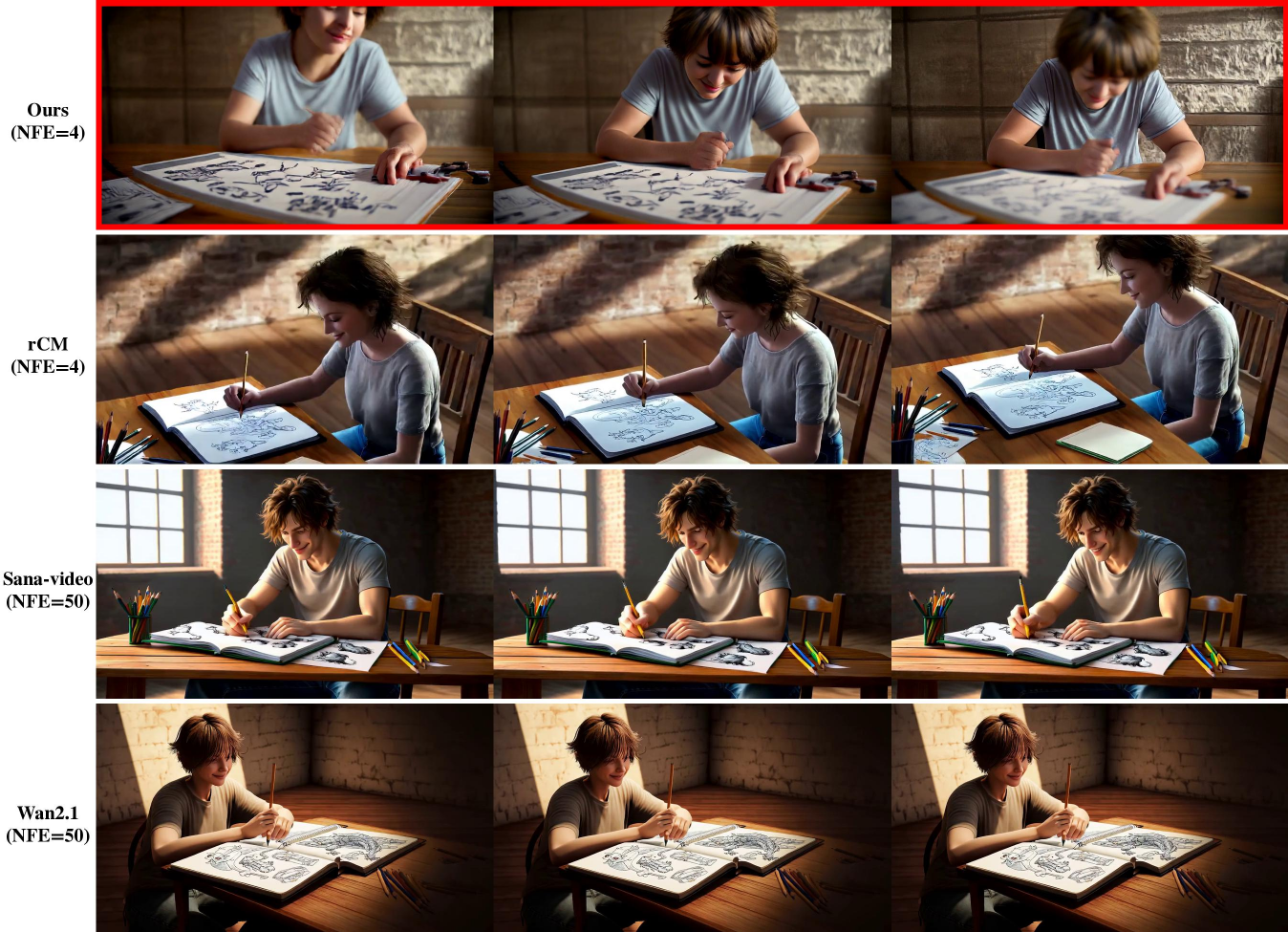}
    \caption{\textbf{Qualitative comparison on 480p text-to-video}. Prompt: \emph{CG game concept digital art, a person sitting at a wooden table, focused intently on a sketchbook. They are wearing comfortable casual clothes, such as a plain t-shirt and jeans. The person has short, messy brown hair and a gentle smile, holding a pencil in their hand. The sketchbook contains detailed drawings of various creatures, including dragons, mythical beasts, and futuristic machines. Soft lighting casts shadows on the walls behind them, creating a cozy atmosphere. The person is drawing with a calm and deliberate pace, occasionally pausing to consult reference materials.}}
    \label{fig:qual_figure6}
\end{figure}

\clearpage

\section{Discussion on Baselines}
\label{app:baselines_discussion}

In this section, we summarize representative baseline architectures that aim to mitigate the computational bottleneck in video diffusion transformers.

\smallskip
\noindent \textbf{Sana-Video}~\cite{chen2025sana} accelerates video diffusion by replacing softmax attention with linear
attention and augmenting locality/temporal aggregation via convolutional mixing (e.g., causal Mix-FFN with temporal
convolutions), reducing attention complexity from quadratic to linear. However, linear attention can trade off
expressiveness, particularly for fine-grained spatial details that are naturally captured by exact softmax attention.
Moreover, Sana-Video’s locality mechanism relies on convolutional mixing, which implicitly assumes a dense neighborhood
layout and can become brittle under random token dropping. In contrast, our GLGA retains a locally expressive
\emph{attention} branch (sliding-window softmax) and is explicitly designed to remain stable under aggressive
\emph{random token dropping}.

\smallskip
\noindent \textbf{Sparse Linear Attention (SLA)}~\cite{zhang2025sla} proposes a hybrid attention mechanism that combines sparse computation with linear/low-rank acceleration by partitioning attention weights into different regimes (e.g., critical, marginal, negligible) and leveraging efficient fused GPU kernels, often with light fine-tuning on large DiTs. A potential limitation is that SLA depends on data-dependent routing and learned attention patterns to determine which interactions to compute sparsely; in early from-scratch training, these patterns can be poorly calibrated, making optimization more brittle in practice (as demonstrated in Fig.~\ref{fig:performance_comparison}).

\smallskip
\noindent \textbf{Video Sparse Attention (VSA)}~\cite{zhang2025vsa} replaces full 3D attention in video DiTs with a trainable sparse-attention pipeline: a coarse stage pools tokens into tiles to adaptively identify salient regions, followed by a fine stage that computes attention only within selected tiles using a hardware-friendly block layout.
A key challenge is that this adaptive selection relies on feature quality; early in training, when representations are still immature, the selection patterns can be poorly calibrated and destabilize optimization. Consistent with this, we observe pronounced training instabilities under standard from-scratch settings (Fig.~\ref{fig:vsa}), which can limit the simplicity and robustness of VSA-style sparsification for from-scratch training. 

\smallskip
\noindent \textbf{NATTEN (Neighborhood Attention Extension)} \cite{hassani2023neighborhood} localizes self-attention to a fixed window, yielding sub-quadratic time/space complexity with a strong local inductive bias, and is supported by custom CUDA kernels in the NATTEN library. This is closely related to our local sliding-window branch; however, window attention alone lacks an explicit global aggregation pathway and typically relies on depth and/or dilation to expand receptive fields. In contrast, GLGA complements windowed softmax attention with a global linear branch and input-aware routing, and our overall system targets both per-step efficiency and few-step scalability.

\smallskip
\noindent \textbf{Native Sparse Attention (NSA)}~\cite{yuan2025native} replaces full softmax attention with a query-dependent sparse mechanism that is trainable end-to-end. NSA constructs three complementary KV sets:
(i) compression tokens that aggregate contiguous KV blocks into coarse representations,
(ii) selection tokens that preserve detail by selecting a small number of contiguous high-importance blocks
(top-$n$), with importance estimated efficiently from compression-level attention scores, and
(iii) a sliding-window branch for local context. These three outputs are subsequently fused via learned gating mechanisms. 
While NSA can achieve high sparsity at inference, integrating it with an aggressive \emph{random token-dropping} training regime is non-trivial. In particular, NSA’s routing depends on forming and scoring \emph{contiguous} KV blocks, whereas dropping 75\% of tokens disrupts this contiguity and can destabilize data-dependent block selection. In contrast, GLGA
does not rely on fragile routing over contiguous blocks. Instead, it maintains explicit, static pathways—a global aggregation branch (linear attention) and a locally expressive branch (windowed softmax)—that remain well-defined and robust even under heavy random token dropping.

\smallskip
\noindent \textbf{S$^2$DiT} \cite{zhao2026s2ditsandwichdiffusiontransformer} targets efficient streaming video generation on mobile by combining efficient attention variants (e.g., LinConv-style hybrid attention and strided self-attention) with a distillation pipeline that transfers capability from a large teacher to a compact few-step student. As a result, S$^2$DiT is optimized for resource-constrained edge settings, where stride-based feature compression and distillation are central to achieving real-time throughput. In contrast, our goal is efficient \emph{from-scratch} training and few-step generation at video scale: rather than permanently compressing the architecture, we leverage compute-efficient training mechanisms (e.g., token dropping and MVA) together with a droppable-by-design attention block, enabling a large and expressive DiT while still reducing training and inference cost.

\smallskip
\noindent{\textbf{Distribution Matching Distillation (DMD)}} \cite{dmd,dmd2,zheng2025large} distills a pre-trained multi-step diffusion model into a one-step generator by matching the student’s output
distribution to the teacher’s distribution, using score-function differences that require additional diffusion models
trained on real versus generated samples. While DMD can yield extremely fast samplers, it is inherently teacher-based and
introduces extra training stages and auxiliary models, which can be particularly costly at video scale. In contrast, we
adopt SoFlow’s teacher-free solution-flow objective and focus on reducing its practical overhead—most notably the extra
SCM and guidance forward passes—via Path-Drop Guided (PDG) training and a droppable-by-design backbone.

\section{Limitations and Future Work}
While \modelname substantially accelerates both training and inference for high-fidelity video generation, several
limitations remain. 
First, although we scale to 480p with a 1.5B-parameter model, industry-scale systems often target
720p/1080p with $10$B$+$ parameters (e.g., Wan2.1-14B or Wan2.2-14B). The behavior of \modelname at these extreme scales and ultra-high resolutions is not yet characterized. Second, training entirely from scratch highlights the raw efficiency of our approach, but inevitably lacks the strong priors (``world knowledge'') provided by large text-to-image foundation models. Consequently, our absolute semantic scores on comprehensive automated benchmarks, such as VBench, currently trail behind those of industrially trained baselines that inherit these broad, pre-trained distributions.

These limitations suggest several directions for future work. A primary goal is to scale \modelname to the 10B+ regime and 1080p resolution to stress-test GLGA and our training recipe under truly large-scale settings. Furthermore, closing the semantic performance gap on open-ended benchmarks like VBench could be addressed by training longer compute or adapting our sparse training framework to support the highly efficient fine-tuning of pre-trained dense models. Finally, while our PDG-SFM training with the MVA regularizer enables stable 4-step generation, extending the framework toward reliably stable 1-step or 2-step sampling remains an important frontier for real-time video synthesis.

\end{document}